%% file: main.tex
\documentclass[lettersize,journal]{IEEEtran}
\usepackage{amsmath,amsfonts}
\usepackage{algorithm}
\usepackage{array}
\usepackage[caption=false,font=normalsize,labelfont=sf,textfont=sf]{subfig}
\usepackage{textcomp}
\usepackage{stfloats}
\usepackage{url}
\usepackage{verbatim}
\usepackage{graphicx}
\usepackage{cite}
\usepackage{amsthm}
\usepackage{mathtools}
\usepackage{amsmath}
\usepackage{xcolor}
\usepackage{multirow}
\usepackage{amssymb}
\DeclareMathOperator{\Var}{Var}
\newtheorem{theorem}{Theorem}
\newtheorem{lemma}{Lemma}
\theoremstyle{remark}
\newtheorem{remark}[theorem]{Remark}

\newcommand{\argmax}[1] {{\operatorname{argmax}}_{#1}} 

\definecolor{myblue}{rgb}{0.87,0.92,0.96}
\usepackage{booktabs}
\usepackage[figuresleft]{rotating}

\usepackage{amsmath}
\usepackage{algpseudocode}

\begin{document}

\title{Hundreds Guide Millions: Adaptive Offline Reinforcement Learning with Expert Guidance}

\author{Qisen Yang*, Shenzhi Wang*, Qihang Zhang*, Gao Huang,~\IEEEmembership{Member,~IEEE}, \\and Shiji Song,~\IEEEmembership{Senior Member,~IEEE}
        % <-this % stops a space
\thanks{Qisen Yang, Shenzhi Wang, Qihang Zhang, Gao Huang, and Shiji Song are with the Department of Automation, Tsinghua University, Beijing 100084, China. Email: \{yangqs19, wsz21, qh-zhang19\}@mails.tsinghua.edu.cn, \{gaohuang, shijis\}@tsinghua.edu.cn. \textit{(Corresponding author: Shiji Song.)}}% <-this % stops a space
\thanks{* Equal contribution.}}

\markboth{IEEE Transactions on Neural Networks and Learning Systems}%
{Shell \MakeLowercase{\textit{et al.}}: A Sample Article Using IEEEtran.cls for IEEE Journals}

% \IEEEpubid{0000--0000/00\$00.00~\copyright~2023  IEEE
% }
\IEEEpubid{\begin{minipage}{\textwidth}\ \\[20pt]
\fbox{
\parbox{0.975\textwidth}{
\copyright 2023 IEEE. Personal use of this material is permitted. Permission from IEEE must be obtained for all other uses, in any current or future media, including reprinting/republishing this material for advertising or promotional purposes, creating new collective works, for resale or redistribution to servers or lists, or reuse of any copyrighted component of this work in other works. 
}}\end{minipage}}

\maketitle

\begin{abstract}
    \input{chapters/0_abstract}
\end{abstract}


\section{Introduction}
\input{chapters/1_introduction}

\section{Preliminaries}
\input{chapters/2_preliminaries}

\section{Method}
\input{chapters/3_method}

\section{Experimental Evaluation}
\input{chapters/4_experiment}

\section{Discussion}
\input{chapters/5_discussion}

\section{Related Work}
\input{chapters/6_related_work}

\section{Conclusion}
\label{sec: conclusion}
\input{chapters/7_conclusion}

\section*{Acknowledgments}
This work is supported in part by the National Science and Technology Major Project of the Ministry of Science and Technology of China under Grants 2019YFC1408703, the National Natural Science Foundation of China under Grants 62022048 and 62276150, the Guoqiang Institute of Tsinghua University.

\appendices
\section{Theorem Proofs}
\label{appendixA}
\input{chapters/8_appendixA}
\section{Implementation Details}
\label{appendixB}
\input{chapters/8_appendixB}

\bibliographystyle{IEEEtran}
\bibliography{ref.bib}

\end{document}

%% file: chapters/0_abstract.tex
Offline reinforcement learning (RL) optimizes the policy on a previously collected dataset without any interactions with the environment, yet usually suffers from the distributional shift problem.
To mitigate this issue, a typical solution is to impose a policy constraint on a policy improvement objective. 
However, existing methods generally adopt a ``one-size-fits-all'' practice, \textit{i.e.}, keeping only a single improvement-constraint balance for all the samples in a mini-batch or even the entire offline dataset.
In this work, we argue that different samples should be treated with different policy constraint intensities.
Based on this idea, a novel plug-in approach named Guided Offline RL (GORL) is proposed.
GORL employs a guiding network, along with only \textit{a few} expert demonstrations, to adaptively determine the relative importance of the policy improvement and policy constraint for every sample.~We theoretically prove that the guidance provided by our method is rational and near-optimal.
Extensive experiments on various environments suggest that GORL can be easily installed on most offline RL algorithms with statistically significant performance improvements.

%% file: chapters/1_introduction.tex
Offline reinforcement learning (RL)~\cite{offline_survey_sergey, offline_survey_prudencio,batch_rl} trains a policy without interactions with the environment.
This characteristic of offline RL brings convenience to applications in many fields where online interactions are expensive or dangerous~\cite{offline_survey_sergey}, such as robotics~\cite{qt-opt, offline-robotics-1, offline-robotics-2, offline-robotics-3}, autonomous driving~\cite{offline-autonomous-driving-1, offline-autonomous-driving-2}, and health care~\cite{offline-health-care-1, offline-health-care-2, offline-health-care-3}.
Nevertheless, offline RL usually suffers from the \textit{distributional shift}~\cite{offline_survey_sergey} problem, due to the gap between state-action distributions of the training dataset and test environment.
Specifically, after optimized on the offline dataset, the agent might encounter unvisited states or misestimate  state-action values during the test in the online environment, leading to a performance collapse.

A prevailing solution~\cite{td3+bc, bear, brac, cql, iql} to the distributional shift problem is reconciling two conflicting objectives: 
(1) \textit{policy improvement}, which is aimed to optimize the policy according to current value functions;
(2) \textit{policy constraint}, which keeps the policy's behavior around the offline dataset to avoid the agent being too aggressive.
Building on this idea, prior methods either add an explicit policy constraint term to the policy improvement equation~\cite{bear, brac, td3+bc}, 
or confine the policy implicitly by revising update rules of value functions~\cite{iql, cql}.
However, these algorithms generally concentrate only on the global characteristics of the dataset, but ignore the individual feature of each sample.
Typically, they make only a single trade-off for all the data in a mini-batch~\cite{td3+bc, bear, brac} or even the whole offline dataset~\cite{cql, iql, bear, brac}.
Such ``one-size-fits-all'' trade-offs might not be able to achieve a perfect balance for each sample, and thus probably limit the potential of~algorithms.

In this work, we argue that, as illustrated in Figure~\ref{fig:teaser}a, a probably ideal improvement-constraint balance for offline RL is to concentrate more on the policy constraint for samples resembling~expert behaviors, but stress more on the policy improvement for data similar to random behaviors. %\shenzhi{high/low-performing}
Furthermore, we notice that expert demonstrations, even in a small quantity, is proved beneficial to the policy performance by many online RL methods~\cite{expert-demo-1, expert-demo-2,9694460}, but few offline algorithms are able to take full advantage of them.
Based on these two observations, we propose to determine an adaptive trade-off between the policy improvement and policy constraint for each sample with the guidance of only \textit{a few} expert data.
As shown in Figure~\ref{fig:teaser}b, the \textit{offline dataset} contains an enormous amount of data, and the \textit{guiding dataset} consists of a few expert demonstrations. 
We alternate between updating the \textit{guiding network} on the guiding dataset in a MAML-like~\cite{maml,9615119} way and training the RL agent on the offline dataset with the guidance of the guiding network.
Our approach points out a theoretically guaranteed optimization direction for the agent and is easy to implement on most offline RL algorithms.

Our main contribution is a plug-in approach, dubbed \textit{Guided Offline RL} (GORL), which determines the relative importance of policy constraints for every sample in an adaptive and end-to-end way.
A possibly surprising finding is that, with the guidance of only \textit{a few} expert demonstrations, GORL achieves significant performance improvement on a number of state-of-the-art offline RL algorithms~\cite{td3+bc, iql, cql} in various tasks of D4RL~\cite{Fu2020D4RLDF}.
Theoretical analyses also validate the rationality and near-optimality of the guidance provided by GORL.

\begin{figure*}[t]
\centering
\includegraphics[width=0.9\textwidth]{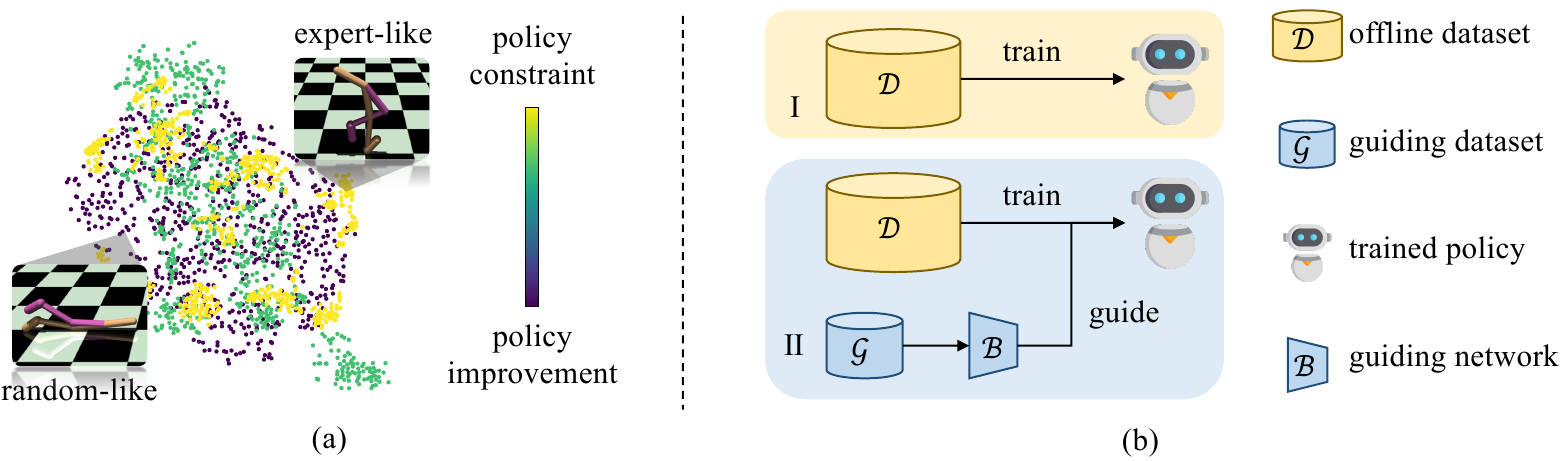}
\vspace{-8pt}
\caption{\textbf{(a)~A possibly ideal improvement-constraint trade-off.} The policy constraint is more encouraged for the data similar to expert behaviors (\textit{i.e.}~``expert-like''). In contrast, the policy improvement is more emphasized for the samples akin to random behaviors (\textit{i.e.}~``random-like'').
The data points are t-SNE~\cite{tsne} visualization of $(s, a, s^\prime, r)$ pairs collected from the Walker2d environment~\cite{mujoco}.
\textbf{(b)~Offline RL (\uppercase\expandafter{\romannumeral1}) and GORL (\uppercase\expandafter{\romannumeral2}).}  Different from the vanilla offline RL, GORL fully utilizes a limited number of expert demonstrations (\textit{i.e.}~the guiding dataset) along with the guiding network and a MAML-like updating method.}
\label{fig:teaser}
\vspace{-10pt}
\end{figure*}

%% file: chapters/2_preliminaries.tex
\IEEEpubidadjcol

In this section, we introduce some basic concepts and notations used in the following sections.

\textbf{RL formulation}.
RL is usually modeled as a Markov decision process denoted as a tuple $(\mathcal{S}, \mathcal{A}, P, d_0, R, \gamma)$,
where $\mathcal{S}$ is the state space, $\mathcal{A}$ is the action space, 
$P(s_{t+1}\mid s_t, a)$ is the environment's state transition probability, 
$d_0(s_0)$ denotes a distribution of the initial state $s_0$, 
$R(s_t, a, s_{t+1})$ defines a reward function,
and $\gamma\in (0, 1]$ is a discount factor.

\textbf{Offline RL}.
Unlike online RL which learns a policy by interacting with the environment, offline RL aims to optimize the policy by only an offline dataset $\mathcal{D} = \{(s_k, a_k, s^{\prime}_k, r_k) \mid k=1, 2, \cdots, N\}$ without any interaction with the environment.
Although there exist various offline RL algorithms with different training losses, their goals are roughly two-fold, either explicitly~\cite{bear, brac, td3+bc} or implicitly~\cite{iql, cql}:
(1) \textit{policy improvement}, which is aimed to optimize the policy according to current value functions;
(2) \textit{policy constraint}, which keeps the policy around the behavior policy or offline dataset's distribution.

Algorithms have to make a trade-off between these two objectives: if concentrating on the policy improvement term too much, the policy probably steps into an unfamiliar area and generates bad actions due to \textit{distributional shift}~\cite{offline_survey_sergey};
otherwise, focusing excessively on the policy constraint term might lead to the policy only imitating behaviors in the offline dataset $\mathcal{D}$ \cite{offline_survey_sergey,    
offline_survey_prudencio} and possibly lacking generalization ability towards out-of-distribution data~\cite{ood1, bear, brac, wang2021glancing}.

In this paper, we install GORL on several state-of-the-art offline RL algorithms including TD3+BC~\cite{td3+bc} and its variant SAC+BC (applying SAC~\cite{sac} to the TD3+BC framework), CQL~\cite{cql}, and IQL~\cite{iql}. 
Their policy optimization objectives can be unified as follows:

\vspace{-10pt}
\begin{small}
\begin{equation}
\label{eq: unified objective}
    \begin{aligned}
    \pi^* =& \argmax{\pi} \mathbb{E}_{(s, a)\sim \mathcal{D}} \operatorname{F} \Big(\underbrace{L_{pi}\left(Q, \pi, s, a\right)}_{\text{policy improvement term}}, \\ 
    & \underbrace{L_{pc}\left(Q, \pi, s, a\right)}_{\text{policy constraint term}}, \underbrace{d_c}_{\text{constraint degree}}\Big), 
    \end{aligned}
\end{equation}
\end{small}
where $\pi: \mathcal{S}\to \mathcal{A}$ is a policy, $\pi^*: \mathcal{S}\to \mathcal{A}$ is the optimal policy, $Q(s, a): \mathcal{S}\times \mathcal{A}\to \mathbb{R}$ is a state-action value function estimating the expected sum of discounted rewards after taking action $a$ at state $s$.

Furthermore, $L_{pi}(\cdot)$ and $L_{pc}(\cdot)$ stand for a policy improvement term and a policy constraint term, and $F(\cdot)$ is a trade-off function between $L_{pi}(\cdot)$ and $L_{pc}(\cdot)$. 
$d_c \in \mathbb{R}$ is a constraint degree:
larger $d_c$ would encourage stronger policy constraints, and therefore the policy becomes more conservative; 
otherwise, the policy would stress the policy improvement, and thus tends to be more aggressive. 

%% file: chapters/3_method.tex
In Section~\ref{sec: APC}, we initially elucidate the manner in which GORL employs guiding data for adapting constraint degrees. Subsequently, the theoretical underpinnings of GORL's update rules are examined in Section~\ref{sec: intuitions behind GORL}. Section~\ref{sec: theoretical proof} offers proof concerning the near-optimality of the guidance afforded by GORL. Lastly, practical instantiations of GORL are expounded upon in Section~\ref{sec: Practical implementations of GORL}.

\subsection{The GORL Framework} 
\label{sec: APC}
Consider an offline RL training problem with an offline dataset $\mathcal{D} = \{(s_k, a_k, s^{\prime}_k, r_k) \mid k=1, 2, \cdots, N\}$ and a guiding dataset $\mathcal{G} = \{(s_k, a_k, s_{k}^{\prime}, r_k)\mid k=1, 2, \cdots, M\}$, where $M \ll N$,  and  $\mathbb{E}_{\tau \sim \mathcal{G}}\left[ R(\tau) \right] > \mathbb{E}_{\tau \sim\mathcal{D}}[R(\tau)]$ ($\tau$ is a trajectory and $R(\tau)$ denotes the cumulative reward of the trajectory $\tau$).
For instance, $\mathcal{D}$ is a large offline dataset containing sub-optimal or even random policies' trajectories, while $\mathcal{G}$ is a guiding dataset with a small quantity of data collected by the expert policies. 

We adopt the training objective of the policy $\pi_{\boldsymbol{\theta}}$ similar to that in TD3+BC~\cite{td3+bc} to demonstrate our method:

\begin{small}
\begin{equation}
\label{eq: algos-BC training objective}
    \begin{aligned}
    \pi^* = &\argmax{\pi_{\boldsymbol{\theta}}} \mathbb{E}_{(s_k, a_k) \sim \mathcal{D}}\Big[ \underbrace{Q(s_k, \pi_{\boldsymbol{\theta}}(s_k))}_{\text{policy improvement term}} \\ &- \underbrace{\mathcal{B}_{\boldsymbol{w}}\left(L_{pc}(a_k, \pi_{\boldsymbol{\theta}}(s_k))\right)}_{\text{constraint degree}} \cdot  \underbrace{L_{pc}(a_k, \pi_{\boldsymbol{\theta}}(s_k))}_{\text{policy constraint term}}\Big],
    \end{aligned}
\end{equation}
\end{small}
where $L_{pc}(\cdot)$ stands for a policy constraint term, \textit{e.g.}, $(\pi_{\boldsymbol{\theta}}(s_k) - a_k)^2$ in TD3+BC~\cite{td3+bc}.
The guiding network $\mathcal{B_{\boldsymbol{w}}}: \mathbb{R}\to \mathbb{R}$ with parameters $\boldsymbol{w}$ takes a policy constraint term $L_{pc}$ as input, and outputs a constraint degree.

It's worth noting in Equation~\eqref{eq: algos-BC training objective} that although inputting the policy constraint term $L_{pc}$ to $\mathcal{B}_{\boldsymbol{w}}$ might sacrifice some information of samples, it additionally incorporate the information of targets. Therefore, taking losses as input is a common practice in machine learning and its validity has been proved by many sample weighting methods~\cite{meta-weight-net,Adaboost,Ensemble_of_exemplar-SVMs_for_object_detection_and_beyond,focal_loss}. 

\textbf{Updating the guiding network $\mathcal{B}_{\boldsymbol{w}}$.}
Here we introduce how to update the guiding network $\mathcal{B}_{\boldsymbol{w}}$ to better balance the policy improvement and policy constraint terms.
Similar to~\cite{maml, meta-weight-net}, we start by updating the policy parameters $\boldsymbol{\theta}$ with  a gradient descent step on the offline dataset $\mathcal{D}$:

\vspace{-10pt}
\begin{small}
\begin{equation}
\label{eq: meta update of policy}
\begin{aligned}
     \hat{\boldsymbol{\theta}}^{(t)}(\boldsymbol{w}^{(t)}) =  \boldsymbol{\theta}^{(t)}  -  \alpha_{\mathcal{D}} \frac{1}{n_{\mathcal{D}}} \sum_{k=1}^{n_{\mathcal{D}}}  &\Big[ -\left.\nabla_{\boldsymbol{\theta}} Q(s_k, \pi_{\boldsymbol{\theta}}(s_k)) \right|_{\boldsymbol{\theta}^{(t)}}
      \\ + \mathcal{B}_{\boldsymbol{w}^{(t)}}\left(L_{pc}(a_k, \pi_{\boldsymbol{\theta}^{(t)}}(s_k))\right) 
      & \cdot \left. \nabla_{\boldsymbol{\theta}} L_{pc}(a_k, \pi_{\boldsymbol{\theta}}(s_k))\right|_{\boldsymbol{\theta}^{(t)}} \Big], \\ \qquad (s_k, a_k)\sim \mathcal{D},
\end{aligned}
\end{equation}
\end{small}
where $n_{\mathcal{D}}$ and $\alpha_{\mathcal{D}}$ are the mini-batch size and learning rate respectively on the offline dataset $\mathcal{D}$.
Based on $\hat{\boldsymbol{\theta}}(\boldsymbol{w})$, the guiding network $\mathcal{B}_{\boldsymbol{w}}$ is updated on the guiding dataset $\mathcal{G}$:

\vspace{-10pt}
\begin{small}
\begin{equation}
    \begin{aligned}
    &\boldsymbol{w}^{(t+1)} = 
    \boldsymbol{w}^{(t)} - \alpha_{\mathcal{G}} \frac{1}{n_{\mathcal{G}}} 
    \sum_{k=1}^{n_{\mathcal{G}}}
    \left.\nabla_{\boldsymbol{w}} L_{pc}(a_k, \pi_{\color{blue}\hat{\boldsymbol{\theta}}^{(t)}(\boldsymbol{w}^{(t)})}(s_k))\right|_{\boldsymbol{w}^{(t)}},\\
    &\qquad\qquad\qquad\qquad(s_k, a_k) \sim \mathcal{G},
    \end{aligned}
    \label{eq: update guiding network}
\end{equation}
\end{small}
where $n_{\mathcal{G}}$ is the mini-batch size and $\alpha_{\mathcal{G}}$ is the step size on the guiding dataset $\mathcal{G}$. 
Note that the policy's parameters marked in {\color{blue}blue} is the updated parameters in Equation~\eqref{eq: meta update of policy} related to $\boldsymbol{w}$.
Because updating on $\mathcal{G}$ with purely expert~demonstrations, Equation~\eqref{eq: update guiding network} performs only behavior cloning without the~policy improvement and constraint degree. 
This will encourage $B_{\boldsymbol{w}}$ to output greater constraint degrees for the gradient directions close to expert imitation, as proved in Theorem~\ref{thm: intuition behind GORL}. 

\textbf{Updating the policy $\pi_{\boldsymbol{\theta}}$.}
After performing a gradient descent step on guiding network $\mathcal{B}_{\boldsymbol{w}}$ in Equation~\eqref{eq: update guiding network}, we further move the policy's parameters $\boldsymbol{\theta}$ toward the direction of maximizing the policy objective in Equation~\eqref{eq: algos-BC training objective}:

\vspace{-10pt}
\begin{small}
\begin{equation}
    \begin{aligned}
     \boldsymbol{\theta}^{(t+1)} = \boldsymbol{\theta}^{(t)}  -  \alpha_{\mathcal{D}} \frac{1}{n_{\mathcal{D}}}
    \sum_{k=1}^{n_{\mathcal{D}}} 
     \Big[ 
    \underbrace{-\left.\nabla_{\boldsymbol{\theta}} Q(s_k, \pi_{\boldsymbol{\theta}}(s_k))\right|_{\boldsymbol{\theta}^{(t)}}}_{\text{policy improvement gradient}}
     + \quad\quad\quad \\ \mathcal{B}_{\color{blue}\boldsymbol{w}^{(t+1)}}\left(L_{pc}(a_k, \pi_{\boldsymbol{\theta}^{(t)}}(s_k))\right) \cdot
    \underbrace{\left.\nabla_{\boldsymbol{\theta}} L_{pc}(a_k, \pi_{\boldsymbol{\theta}}(s_k))\right|_{\boldsymbol{\theta}^{(t)}}}_{\text{policy constraint gradient}}
    \Big], \quad \\ (s_k, a_k)\sim \mathcal{D}. \quad\quad\quad\quad\quad\quad\quad\quad\quad\quad\quad\quad\quad\quad\quad\quad\quad\quad\quad
    \end{aligned}
\label{eq: update policy}
\end{equation}
\end{small}

Note that the place marked in {\color{blue} blue} in Equation~\eqref{eq: update policy} is $\boldsymbol{w}^{(t+1)}$ instead of $\boldsymbol{w}^{(t)}$ in Equation~\eqref{eq: meta update of policy}.
It can be clearly observed in Equation~\eqref{eq: update policy} that the guiding network $\mathcal{B}_{\boldsymbol{w}}$ controls the relative update steps of policy improvement and policy constraint gradient for each data pair $(s_k, a_k)$ in the mini-batches.

\subsection{Theoretical Analysis of GORL}  \label{sec: analysis of GORL}

In this section, we first demonstrate the rationality of GORL's update mechanism in Section~\ref{sec: intuitions behind GORL}, and then theoretically prove the near-optimality of GORL's guiding gradient average in Section~\ref{sec: theoretical proof}. 

\subsubsection{Rationality of GORL's Update Mechanism} \label{sec: intuitions behind GORL}

\begin{theorem} \label{thm: intuition behind GORL}
By the chain rule, Equation~\eqref{eq: update guiding network} can be reformulated as:
\begin{small}    
\begin{equation}    
\label{eq: meta update after the chain rule}
    \boldsymbol{w}^{(t+1)} = \boldsymbol{w}^{(t)} + \frac{\alpha_{\mathcal{D}} \alpha_{\mathcal{G}}}{n_{\mathcal{D}}}
    \sum_{{k_1}=1}^{n_{\mathcal{D}}} C_{k_1} \left.\frac{\partial \mathcal{B}_{\boldsymbol{w}}(L^{\text{policy}}_{k_1}(\boldsymbol{\theta}^{(t)}))}{\partial \boldsymbol{w}}\right|_{\boldsymbol{w}^{(t)}},
\end{equation}
\end{small}
where
\begin{small}
\begin{equation} 
\centering
\label{eq: ck}
    \begin{aligned}
    \qquad C_{k_1}=\underbrace{\left( \frac{1}{n_{\mathcal{G}}}\sum_{k_2=1}^{n_{\mathcal{G}}}\left.\frac{\partial L^{\text{guide}}_{k_2}(\boldsymbol{\theta})}{\partial \boldsymbol{\theta}}\right|_{\hat{\boldsymbol{\theta}}^{(t)}(\boldsymbol{w}^{(t)})} \right)^{\top}}_{\text{guiding gradient average}}
     \left.\frac{\partial L^{\text{policy}}_{k_1}(\boldsymbol{\theta})}{\partial \boldsymbol{\theta}}\right|_{\boldsymbol{\theta}^{(t)}}.\quad\quad\quad\quad
    \end{aligned}
\end{equation}
\end{small}
Here, the policy's loss
$L_{k_1}^{\text{policy}}(\boldsymbol{\theta}) =  L_{pc}(a_{k_1}, \pi_{\boldsymbol{\theta}}(s_{k_1}))$ with $(s_{k_1}, a_{k_1})\sim \mathcal{D}$,
and the guiding loss $L_{k_2}^{\text{guide}}(\boldsymbol{\theta})
= L_{pc}(a_{k_2}, \pi_{\boldsymbol{\theta}}(s_{k_2}))$ with $(s_{k_2}, a_{k_2})\sim \mathcal{G}$.
\end{theorem}
The proof, inspired by \cite{meta-weight-net}, can be found in Appendix~\ref{appendix: proof of the intuition of GORL}.

It can be observed that in Equation~\eqref{eq: meta update after the chain rule}, 
larger $C_{k_1}$ would encourage the guiding network $\mathcal{B}_{\boldsymbol{w}}$ to output a larger constraint degree for the corresponding policy's loss $L_{k_1}^{\text{policy}}(\boldsymbol{\theta}^{(t)})$.
Further note that in Equation~\eqref{eq: ck}, $C_{k_1}$ is an inner product between the guiding gradient average and policy's gradient  $\frac{\partial L_{k_1}^{\text{policy}}(\boldsymbol{\theta})}{\partial \boldsymbol{\theta}}$.  
Therefore, $\mathcal{B}_{\boldsymbol{w}}$ would assign larger weights for those $L_{k_1}^{\text{policy}}(\boldsymbol{\theta})$ whose gradients are close to the guiding gradient average.
This mechanism is consistent with why MAML~\cite{maml} or Meta-Weight-Net~\cite{meta-weight-net} functions well.

The effects are two-fold. 
Firstly, the policy would align its update directions closer to the guiding gradient average~\cite{meta-weight-net}. According to Theorem~\ref{thm: gap between meta and optimal gradient} below, the guiding gradient average converges to the optimal update gradient in probability, so the gradient alignment would lead to  better update directions for the policy;
Secondly, besides the guidance from the guiding dataset, the policy could also enjoy plenty of environmental information provided by a large number of data in the offline dataset $\mathcal{D}$, which is scarce in the guiding dataset $\mathcal{G}$ due to its small data quantity.

\subsubsection{Near-Optimality of the Guidance from GORL} 
\label{sec: theoretical proof}

To demonstrate that the guiding gradient average $\frac{\partial L^{\operatorname{guide}}}{\partial \hat{\boldsymbol{\theta}}} = \frac{1}{n_{\mathcal{G}}}\sum_{k_2=1}^{n_{\mathcal{G}}}\left.\frac{\partial L^{\text{guide}}_{k_2}(\boldsymbol{\theta})}{\partial \boldsymbol{\theta}}\right|_{\hat{\boldsymbol{\theta}}^{(t)}(\boldsymbol{w}^{(t)})}$ in Equation~\eqref{eq: ck} is qualified for guiding the offline training, we denote the guiding gradient obtained on $n$ expert guiding data as  $\frac{\partial L_{1:n}^{\operatorname{guide}}(\hat{\boldsymbol{\theta}})}{\partial \hat{\boldsymbol{\theta}}}$. Formally,

\begin{small}
\begin{equation}
    \frac{\partial L_{1:n}^{\operatorname{guide}}(\hat{\boldsymbol{\theta}})}{\partial \hat{\boldsymbol{\theta}}} \coloneqq \frac{1}{n}\sum_{k=1}^n \frac{\partial L_{k}^{\operatorname{guide}}(\hat{\boldsymbol{\theta}})}{\partial \hat{\boldsymbol{\theta}}}.
\end{equation}
\end{small}

If the number of guiding data tends to infinity, the training process on $\mathcal{G}$ is behavior cloning on infinite expert data. 
Therefore the guiding gradient average will become the optimal update gradient:

\begin{small}
\begin{equation}
    \begin{aligned}
\frac{\partial L^{\operatorname{guide}}_{*}}{\partial \hat{\boldsymbol{\theta}}} \coloneqq \lim_{n\to\infty} \frac{\partial L_{1:n}^{\operatorname{guide}}(\hat{\boldsymbol{\theta}})}{\partial \hat{\boldsymbol{\theta}}}  = \lim_{n\to \infty} \mathbb{E}_{k\sim \operatorname{unif}\{1, n\}}\left[ \frac{\partial L_{k}^{\operatorname{guide}}(\hat{\boldsymbol{\theta}})}{\partial \hat{\boldsymbol{\theta}}} \right].
    \end{aligned}
\end{equation}
\end{small}

Theorem~\ref{thm: gap between meta and optimal gradient} shows that when $n$ increases, the guiding gradient
% on $n$ expert guiding data 
$\frac{\partial L^{\operatorname{guide}}_{1:n}}{\partial \hat{\boldsymbol{\theta}}}$ will converge to the optimal update gradient $\frac{\partial L^{\operatorname{guide}}_*}{\partial \hat{\boldsymbol{\theta}}}$ in probability at a rate  $\ge \frac{1}{n}$.

\begin{theorem}(Near-optimality of the guiding gradient average) \label{thm: gap between meta and optimal gradient}
Here we analyze the near-optimality of the gradient in each layer.
Suppose $\hat{\boldsymbol{\theta}}^{[l]} \in \mathbb{R}^{d_1\times d_2}$ is the trainable parameters of the $l$-th layer, and the elements in $\frac{\partial L_{k}^{\operatorname{guide}}(\hat{\boldsymbol{\theta}})}{\partial \hat{\boldsymbol{\theta}}^{[l]}} \ (k=1, 2, \cdots, n)$ are independent with their variances bounded by $\delta$.
Then the gap  between the $l$-th layer's  guiding gradient average on $n$ expert guiding data $\frac{\partial L^{\operatorname{guide}}_{1:n}}{\partial \hat{\boldsymbol{\theta}}^{[l]}}$ and the  $l$-th layer's optimal update gradient $\frac{\partial L^{\operatorname{guide}}_*}{\partial \hat{\boldsymbol{\theta}}^{[l]}}$, \textit{i.e.}, $\left\| \frac{\partial L^{\operatorname{guide}}_{1:n}}{\partial \hat{\boldsymbol{\theta}}^{[l]}} -\frac{\partial L^{\operatorname{guide}}_*}{\partial \hat{\boldsymbol{\theta}}^{[l]}} \right\|_1$, is $O_p(\frac{1}{n})$. More specifically,

\vspace{-10pt}
\begin{small}
\begin{equation}
    \forall \epsilon > 0, \quad
    P\left(\left\| \frac{\partial L^{\operatorname{guide}}_{1:n}}{\partial \hat{\boldsymbol{\theta}}^{[l]}} -\frac{\partial L^{\operatorname{guide}}_*}{\partial \hat{\boldsymbol{\theta}}^{[l]}} \right\|_1 \ge \epsilon\right) \le \frac{d_1d_2 \delta}{\epsilon^2} \frac{1}{n}.
\end{equation}
\end{small}
Therefore, $\frac{\partial L^{\operatorname{guide}}_{1:n}}{\partial \hat{\boldsymbol{\theta}}^{[l]}}$ converges to $\frac{\partial L^{\operatorname{guide}}_*}{\partial \hat{\boldsymbol{\theta}}^{[l]}}$ in probability at a rate $\ge \frac{1}{n}$, \textit{i.e.},

\vspace{-10pt}
\begin{small}
\begin{equation}
\begin{aligned}
     \forall \epsilon > 0, \quad
    &0 \le \lim_{n\to\infty} P\left(\left\| \frac{\partial L^{\operatorname{guide}}_{1:n}}{\partial \hat{\boldsymbol{\theta}}^{[l]}} -\frac{\partial L^{\operatorname{guide}}_*}{\partial \hat{\boldsymbol{\theta}}^{[l]}} \right\|_1 \ge \epsilon\right)\quad\quad \\ &\quad\quad\quad\quad\quad\quad\quad\quad\quad\quad\quad\quad\quad\le \frac{d_1d_2 \delta}{\epsilon^2} \lim_{n\to \infty}\frac{1}{n}  = 0, 
    \\ &\Longrightarrow
    \frac{\partial L^{\operatorname{guide}}_{1:n}}{\partial \hat{\boldsymbol{\theta}}^{[l]}} \overset{p}{\to} \frac{\partial L^{\operatorname{guide}}_*}{\partial \hat{\boldsymbol{\theta}}^{[l]}} \quad \text{(at a rate $\ge \frac{1}{n}$)}.\quad\quad\quad\quad\quad\quad \nonumber
\end{aligned}
\end{equation}
\end{small}
\end{theorem}

The proof of Theorem~\ref{thm: gap between meta and optimal gradient} is deferred to Appendix~\ref{appendix: proof of the guiding gradient gap}.

\begin{remark}
Theorem~\ref{thm: gap between meta and optimal gradient}'s supposition that $\frac{\partial L_{k}^{\operatorname{guide}}(\hat{\boldsymbol{\theta}})}{\partial \hat{\boldsymbol{\theta}}^{[l]}} \ (k=1, 2, \cdots, n)$ are independent is valid since, given $i$-th layer parameters $\hat{\boldsymbol{\theta}}^{[l]}$, the transformation from the $k$-th expert guiding sample to its associated gradient is deterministic. Additionally, as expert guiding data is i.i.d., so are $\frac{\partial L_{k}^{\operatorname{guide}}(\hat{\boldsymbol{\theta}})}{\partial \hat{\boldsymbol{\theta}}^{[l]}} \ (k=1, \cdots, n)$. Further note that gradient computation for a trainable parameter is independent of other parameters in the same layer, so $\forall k\in \{1, \cdots, n\}, i\in \{1, \cdots, d_1\}, j \in \{1, \cdots, d_2\}$, $\frac{\partial L_{k}^{\operatorname{guide}}(\hat{\boldsymbol{\theta}})}{\partial \hat{\boldsymbol{\theta}}^{[l]}_{ij}}$ are independent.
\end{remark}

Theorem~\ref{thm: gap between meta and optimal gradient} demonstrates that, when the guiding dataset $\mathcal{G}$ has sufficient expert data, the guiding gradient average in Equation~\eqref{eq: ck} will be approximate to the optimal gradient, and therefore provides reliable guidance for the offline RL algorithms. 
Empirically, as shown by the green bars in Figure~\ref{expertsize}a, a quite small quantity of expert data, \textit{e.g.}, $100$~(the offline dataset's size is typically 1 million), is sufficient for the guiding dataset $\mathcal{G}$ to generate a good enough guiding gradient average in Equation~\eqref{eq: ck}.

\subsection{Practical implementations of GORL.} 
\label{sec: Practical implementations of GORL}

The pseudo-code of our proposed plug-in framework, \textit{i.e.}, GORL, is presented in Algorithm~\ref{alg: algos-BC}.
Furthermore, we provide examples of how to apply GORL to some popular offline RL algorithms, including TD3+BC~\cite{td3+bc} and its variant SAC+BC, IQL~\cite{iql} and CQL~\cite{cql}.
To implement GORL on offline RL algorithms, one of the most important things is to find out the corresponding policy constraint term.
Such constraint term is explicit in some methods (\textit{e.g.}, TD3+BC~\cite{td3+bc} and SAC+BC), while much more implicit in other algorithms  (\textit{e.g.}, CQL~\cite{cql} and IQL~\cite{iql}). 

\begin{algorithm}
\small
\caption{GORL Algorithm} \label{alg: algos-BC}
\begin{algorithmic}[1]
    \Require Offline dataset $\mathcal{D}$, guiding dataset $\mathcal{G}$, 
    batch sizes $n_{\mathcal{D}}, n_{\mathcal{G}}$, 
    learning rates $\alpha_{\mathcal{D}}, \alpha_{\mathcal{G}}$,
    and training steps $N_{\text{train}}$.
    \Ensure Policy $\pi_{\boldsymbol{\theta}}$ after optimization.
    \State Initialize policy $\pi_{\boldsymbol{\theta}}$ and guiding network $\mathcal{B}_{\boldsymbol{w}}$.
    \For{$t = 1 \to N_{\text{train}}$}
        \State Sample a mini-batch $B_{\text{off}}^{(t)} = \{(s_k, a_k, s_{k}^{\prime}, r_k) \mid k=1, 2, \cdots, n_{\mathcal{D}}\}$ uniformly from $\mathcal{D}$.
        \State Update $Q$ with $B_{\text{off}}^{(t)}$. 
            \State Get $\hat{\boldsymbol{\theta}}^{(t)}(\boldsymbol{w}^{(t)})$ with $B_{\text{off}}^{(t)}$ by Equation~\eqref{eq: meta update of policy}.     
            % \Statex \Comment{Preparation for the guiding network's update.}
            %
            \State Sample a mini-batch $B_{\text{guide}}^{(t)} = \{(s_k, a_k, s_{k}^{\prime}, r_k) \mid k=1, 2, \cdots, n_{\mathcal{G}}\}$ uniformly from $\mathcal{G}$.  
            \State Update $\mathcal{B}_{\boldsymbol{w}^{(t)}}$ to   $\mathcal{B}_{\boldsymbol{w}^{(t+1)}}$ with $B_{\text{guide}}^{(t)}$ by Equation~\eqref{eq: update guiding network}.   
            \State Update $\pi_{\boldsymbol{\theta}^{(t)}}$ to $\pi_{\boldsymbol{\theta}^{(t+1)}}$ with $B_{\text{off}}^{(t)}$ by Equation~\eqref{eq: update policy}.  
    \EndFor
\end{algorithmic}
\end{algorithm}

\textbf{Implementation on TD3+BC~\cite{td3+bc}.}
To implement GORL on TD3+BC~\cite{td3+bc}, one can follow the procedures in Algorithm 1 with $L_{pc}(a_k, \pi_{\boldsymbol{\theta}})$ substituted with $(\pi_{\boldsymbol{\theta}}(s_k) - a_k)^2$.

\textbf{Implementation on SAC+BC.}
SAC+BC is a natural extension of TD3+BC~\cite{td3+bc}, replacing TD3~\cite{td3} with SAC~\cite{sac}. Its policy optimization objective is as below:
\begin{equation}
\begin{aligned}
% \vspace{3pt}
    \pi_{\boldsymbol{\theta}} = &\argmax{{\pi_{\boldsymbol{\theta}}}} \mathbb{E}_{k\sim \text{unif}\{1, N\}} \Big[ Q(s_k, \tilde{a}_k) \underbrace{- \alpha \log \pi_{\boldsymbol{\theta}}(\tilde{a}_k\mid s_k)}_{\text{maximizing entropy term}} \\ &- \mathcal{B}_{\boldsymbol{w}}\left((\tilde{a}_k - a_k)^2\right) \cdot (\tilde{a}_k - a_k)^2 \Big],
% \vspace{3pt}
\end{aligned}
\end{equation}
where $(s_k, a_k)\sim \mathcal{D}$, $\alpha\in \mathbb{R}$ is a constant, $\tilde{a}_k$ is an action sampled from $\pi_{\boldsymbol{\theta}}(\cdot\mid s_k)$ by the reparameterization trick~\cite{sac}, and $\pi_{\boldsymbol{\theta}}(\tilde{a}_k \mid s_k)$ denotes the probability of $\pi_{\boldsymbol{\theta}}$ choosing $\tilde{a}_k$ at state $s_k$.
The Q-function optimization of SAC+BC is the same as SAC~\cite{sac}.
By adding the maximizing entropy term to Equation (3), (4), (5), GORL can be applied to SAC+BC with the procedures in Algorithm 1.

\textbf{Implementation on IQL~\cite{iql}.}
The policy update objective in IQL~\cite{iql} is:
\begin{equation}
\begin{aligned}
    \label{eq: IQL original objective}
    \pi_{\boldsymbol{\theta}} = &\argmax{\boldsymbol{\theta}} \mathbb{E}_{(s, a) \sim \mathcal{D}}[\exp (\beta(Q(s, a)-V(s))) \\ &\cdot \log \pi_{\boldsymbol{\theta}}(a\mid s)],
\end{aligned}
\end{equation}

where $V(s)$ is an approximator of $\mathbb{E}_{a}\left[Q(s, a)\right]$.
The intuition behind Equation~\eqref{eq: IQL original objective} is that, if some action $a_k$ is in advantage, \textit{i.e.}, $Q(s, a_k) > \mathbb{E}_a\left[Q(s, a)\right] = V(s_k)$, the term $\exp \left(\beta\left(Q(s, a_k)-V(s)\right)\right)$ will be larger than its expectation $\mathbb{E}_a\left[\exp \left(\beta\left(Q(s, a)-V(s)\right)\right)\right]$.
Therefore, after updating with Equation~\eqref{eq: IQL original objective},  $\pi_{\boldsymbol{\theta}}$ is more likely to choose $a_k$ rather than other actions.
It's obvious that scalar $\beta$, together with $Q(s, a) - V(s)$, controls to what extent $\pi_{\boldsymbol{\theta}}$ accepts action $a$ at state $s$.

Based on the observation above, we reformulate Equation~\eqref{eq: IQL original objective} into:
\begin{equation}
    \label{eq: IQL individual beta}
\begin{aligned}
    \pi_{\boldsymbol{\theta}} = &\argmax{\boldsymbol{\theta}} \mathbb{E}_{k\sim \text{unif}\{1, N\}}[\exp \left(\beta_k\left(Q(s_k, a_k)-V(s_k)\right)\right) 
    \\ &\log \pi_{\boldsymbol{\theta}}(a_k\mid s_k)], \quad (s_k, a_k)\sim \mathcal{D}.
\end{aligned}
\end{equation}
Compared with Equation~\eqref{eq: IQL original objective}, Equation~\eqref{eq: IQL individual beta} assigns a different scalar $\beta_k$ for each data pair $(s_k, a_k)$.
However, note that $\exp \left(\beta_k\left(Q(s_k, a_k)-V(s_k)\right)\right) = \left( \exp \left(Q(s_k, a_k)-V(s_k)\right)\right)^{\beta_k}$.
It's difficult to find an optimal $\beta_k$ end-to-end because $\beta_k$ is the exponent of $\left( \exp \left(Q(s_k, a_k)-V(s_k)\right)\right)$. 

To make the optimization of $\beta_k$ possible, we further change Equation~\eqref{eq: IQL individual beta} into:
\begin{equation}
    \label{eq: IQL new objective}
\begin{aligned}
    \pi_{\boldsymbol{\theta}} = &\argmax{\boldsymbol{\theta}} \mathbb{E}_{k\sim \text{unif}\{1, N\}}[\beta_k \exp (Q(s_k, a_k)-V(s_k)) \\ &\log \pi_{\boldsymbol{\theta}}(a_k\mid s_k)], \quad (s_k, a_k)\sim \mathcal{D},
\end{aligned}
\end{equation}
where $\beta_k$ is multiplied by $\exp \left(Q(s_k, a_k)-V(s_k)\right)$, which is much easier to optimize. 

GORL can be implemented on IQL following Algorithm 1 changed with the new objective (Equation~\eqref{eq: IQL new objective}).
$\mathcal{B}_{\boldsymbol{w}}$ is used to generate $\beta_k$ by taking $\log \pi_{\boldsymbol{\theta}}(\tilde{a}_k\mid s_k)$  as input.

\textbf{Implementation on CQL~\cite{cql}.}
The policy update objective in CQL~\cite{cql} is:
\begin{equation}
\vspace{3pt}
    \label{eq: CQL original objective}
    \pi_{\boldsymbol{\theta}} = \argmax{\boldsymbol{\theta}} \mathbb{E}_{(s, a) \sim \mathcal{D}}\left[Q(s,a) - \log \pi_{\boldsymbol{\theta}}(a\mid s)\right],
\end{equation}
where $Q(s,a)$ is a conservative approximation of the state-action value.
The policy constraint objective is implicitly contained during the conservative Q-learning.
The more conservative Q-value represents the stronger policy constraint.
In this case, GORL can be implemented following Algorithm 1 with the new policy update objective below: 
\begin{equation}
% \vspace{3pt}
\begin{aligned}
    \label{eq: CQL original objective}
    \pi_{\boldsymbol{\theta}} = &\argmax{\boldsymbol{\theta}} \mathbb{E}_{(s, a) \sim \mathcal{D}}[\mathcal{B}_{\boldsymbol{w}}\left(Q(s,a)\right) \cdot Q(s,a) \\ &- \log \pi_{\boldsymbol{\theta}}(a\mid s)].
\end{aligned}
\end{equation}

%% file: chapters/4_experiment.tex
\begin{table*}[t]
  \small
  \caption{Average normalized score over the final 10 evaluations and 5 seeds. 
  GORL achieves great performance improvement on all the 4 state-of-the-art algorithms in the locomotion and adroit tasks.
  }
  \centering
  \begin{tabular}{l|cc|cc|cc|cc||cc}
    \specialrule{0.12em}{0pt}{0pt}
	\multirow{2}{*}{Dataset}
	& \multicolumn{2}{c|}{TD3+BC~\cite{td3+bc}} 
	& \multicolumn{2}{c|}{SAC+BC~\cite{td3+bc,sac}} 
	& \multicolumn{2}{c|}{IQL~\cite{iql}} 
	& \multicolumn{2}{c||}{CQL~\cite{cql}}
	& \multicolumn{2}{c}{\textbf{Avg.}}
	\\ 
	\cline{2-11}
	& Base & Ours
	& Base & Ours
	& Base & Ours
	& Base & Ours
	& Base & Ours
	\\ \hline
    halfcheetah-random-v2 
    & 11.2 & \colorbox{myblue}{16.5} & 12.9 & \colorbox{myblue}{15.6}
    & 10.6 & \colorbox{myblue}{12.2} & \colorbox{myblue}{22.3} & 20.3 
    & 14.3 
    & \textbf{16.2} 
    \\
    hopper-random-v2 
    & 8.8 & \colorbox{myblue}{16.8} & 7.9 & \colorbox{myblue}{28.2}
    & 7.9 & \colorbox{myblue}{8.4} & 8.7 & \colorbox{myblue}{9.8} 
    & 8.3 
    & \textbf{15.8} 
    \\
    walker2d-random-v2 
    & 1.3 & \colorbox{myblue}{2.3} & 1.6 & \colorbox{myblue}{2.6}
    & \colorbox{myblue}{6.6} & 6.3 & 1.7 & \colorbox{myblue}{3.4}  
    & 2.8 
    & \textbf{3.6} 
    \\
    halfcheetah-medium-v2 
    & 48.2 & \colorbox{myblue}{51.8} & 49.9 & \colorbox{myblue}{53.6}
    & 46.8 & \colorbox{myblue}{50.2} & 48.9 & \colorbox{myblue}{51.3}  
    & 48.5 
    & \textbf{51.7} 
    \\
    hopper-medium-v2 
    & 59.8 & \colorbox{myblue}{67.1} & 47.9 & \colorbox{myblue}{49.1}
    & 63.9 & \colorbox{myblue}{68.4} & 70.6 & \colorbox{myblue}{72.2}  
    & 60.6 
    & \textbf{64.2} 
    \\  
    walker2d-medium-v2 
    & 83.9 & \colorbox{myblue}{85.7} & 84.6 & \colorbox{myblue}{86.3}
    & 77.3 & \colorbox{myblue}{78.5} & 83.1 & \colorbox{myblue}{83.5}  
    & 82.2 
    & \textbf{83.5} 
    \\ 
    halfcheetah-medium-replay-v2 
    & 44.6 & \colorbox{myblue}{46.7} & 44.7 & \colorbox{myblue}{46.8}
    & 43.9 & \colorbox{myblue}{45.6} & 37.4 & \colorbox{myblue}{47.0}  
    & 42.6 
    & \textbf{46.5} 
    \\ 
    hopper-medium-replay-v2 
    & 61.7 & \colorbox{myblue}{74.7} & 45.8 & \colorbox{myblue}{55.4}
    & 70.9 & \colorbox{myblue}{74.1} & 93.4 & \colorbox{myblue}{94.6}  
    & 67.9 
    & \textbf{74.7} 
    \\ 
    walker2d-medium-replay-v2 
    & 79.5 & \colorbox{myblue}{84.5} & 78.2 & \colorbox{myblue}{84.4}
    & 67.7 & \colorbox{myblue}{68.0} & \colorbox{myblue}{82.8} & 80.4  
    & 77.0 
    & \textbf{79.3} 
    \\  
    halfcheetah-medium-expert-v2 
    & 91.8 & \colorbox{myblue}{95.6} & 89.4 & \colorbox{myblue}{90.8}
    & 80.9 & \colorbox{myblue}{85.3} & 65.9 & \colorbox{myblue}{81.1}  
    & 82.0 
    & \textbf{88.2} 
    \\  
    hopper-medium-expert-v2 
    & 98.6 & \colorbox{myblue}{105.3} & 93.9 & \colorbox{myblue}{94.7}
    & 39.8 & \colorbox{myblue}{49.8} & 100.1 & \colorbox{myblue}{105.3}  
    & 83.1 
    & \textbf{88.8} 
    \\ 
    walker2d-medium-expert-v2 
    & \colorbox{myblue}{110.2} & 109.5 & 110.1 & \colorbox{myblue}{110.6}
    & 108.4 & \colorbox{myblue}{109.2} & 108.8 & \colorbox{myblue}{109.2}
    & 109.4
    & \textbf{109.6} 
    \\  
    \hline
    \textbf{locomotion-v2 total} 
    & 699.5 & \colorbox{myblue}{756.5} & 666.9 & \colorbox{myblue}{718.1}
    & 624.7 & \colorbox{myblue}{655.9} & 723.9 & \colorbox{myblue}{758.1} 
    & 678.8 
    & \textbf{722.1} 
    \\ 
    \hline
    pen-human-v1 
    & 48.6 & \colorbox{myblue}{71.9} & 35.6 & \colorbox{myblue}{83.1} 
    & 71.5 & \colorbox{myblue}{86.1} & 37.5 & \colorbox{myblue}{51.5}  
    & 48.3 
    & \textbf{73.1} 
    \\ 
    hammer-human-v1 
    & 1.5 & \colorbox{myblue}{1.7} & 1.6 & \colorbox{myblue}{2.2}
    & 1.4 & \colorbox{myblue}{1.6} & 4.4 & \colorbox{myblue}{5.6}  
    & 2.2 
    & \textbf{2.8} 
    \\ 
    door-human-v1 
    & -0.1 & \colorbox{myblue}{0.0} & \colorbox{myblue}{-0.1} & \colorbox{myblue}{-0.1}
    & 4.3 & \colorbox{myblue}{5.7} & 9.9 & \colorbox{myblue}{11.6}  
    & 3.5
    & \textbf{4.3} 
    \\ 
    relocate-human-v1 
    & 0.0 & \colorbox{myblue}{0.1} & \colorbox{myblue}{0.1} & \colorbox{myblue}{0.1}
    & 0.1 & \colorbox{myblue}{0.7} & \colorbox{myblue}{0.2} & \colorbox{myblue}{0.2}  
    & 0.1 
    & \textbf{0.2} 
    \\ 
    pen-cloned-v1 
    & 38.6 & \colorbox{myblue}{46.8} & 23.4 & \colorbox{myblue}{39.7} 
    & 37.3 & \colorbox{myblue}{78.8} & 39.2 & \colorbox{myblue}{64.1}  
    & 34.6 
    & \textbf{57.3} 
    \\ 
    hammer-cloned-v1 
    & \colorbox{myblue}{3.6} & 3.5 & \colorbox{myblue}{1.1} & 0.8
    & \colorbox{myblue}{2.1} & 1.9 & \colorbox{myblue}{2.1} & 1.6  
    & \textbf{2.2} 
    & 1.9
    \\ 
    door-cloned-v1 
    & -0.1 & \colorbox{myblue}{0.0} & -0.1 & \colorbox{myblue}{0.0}
    & 1.6 & \colorbox{myblue}{4.4} & 0.4 & \colorbox{myblue}{1.3}  
    & 0.5 
    & \textbf{1.4} 
    \\ 
    relocate-cloned-v1 
    & -0.2 & \colorbox{myblue} {0.0} & \colorbox{myblue}{-0.2} & \colorbox{myblue}{-0.2}
    & -0.2 & \colorbox{myblue}{0.1} & -0.1 & \colorbox{myblue}{0.0}  
    & -0.2 
    & \textbf{0.0} 
    \\ 
    \hline
    \textbf{adroit-v1 total} 
    & 92.1 & \colorbox{myblue}{123.9} & 61.5 & \colorbox{myblue}{125.5}
    & 118.1 & \colorbox{myblue}{179.2} & 93.6 & \colorbox{myblue}{135.8}  
    & 91.3 
    & \textbf{141.1} 
    \\ 
    \hline \hline
    \textbf{total} 
    & 791.6 & \colorbox{myblue}{880.4} & 728.4 & \colorbox{myblue}{843.6}
    & 742.8 & \colorbox{myblue}{835.1} & 817.5 & \colorbox{myblue}{893.9} 
    & 770.1 
    & \textbf{863.3} 
    \\
    \specialrule{0.12em}{0pt}{0pt}
  \end{tabular}
  \label{main_results}
\end{table*}

In this section, we study performance improvement brought by GORL on offline RL algorithms. 

\textbf{Baselines.}~We consider several state-of-the-art methods as baselines, including TD3+BC~\cite{td3+bc}, SAC+BC (a variant of TD3+BC substituting TD3~\cite{td3} with SAC~\cite{sac}), CQL~\cite{cql}, and IQL~\cite{iql}. 
The former two algorithms adopt explicit policy constraints based on behavior cloning, while the latter two implicitly control the distributional shift by learning conservative Q values and avoiding querying unseen actions respectively. 
We follow the author-provided implementations of the above~methods. 

\textbf{Datasets.}~GORL is evaluated on the Gym locomotion~\cite{mujoco,openai_gym} and Adroit robotic manipulation tasks~\cite{adroit} in the D4RL benchmark~\cite{Fu2020D4RLDF}.
Locomotion includes datasets in various environments with different qualities.
Adroit considers controlling a 24-DoF robotic hand to complete several tasks and collects datasets from various sources.
The standard offline dataset contains approximately 100 thousand or 1 million $(s_k, a_k, s^{\prime}_k, r_k)$ tuples. 
For both locomotion and adroit, we randomly select 200 tuples from their official expert datasets as the guiding data $\mathcal{G}$.
More training details are given in Appendix~\ref{appendixB}.

\textbf{Comparisons.}~Performances of the state-of-the-art offline RL algorithms w/o and w/ GORL are presented in Table~\ref{main_results}.
Generally, our method obtains significant and stable performance improvements on various tasks and various algorithms.
From the Avg. column in Table~\ref{main_results}, GORL surpasses the vanilla algorithms on all the locomotion datasets.
As for adroit tasks, existing methods typically struggle to learn reasonable policies on some challenging datasets.
Although limited by its base methods' performances, GORL still achieves substantially higher scores on the pen-human and pen-cloned datasets.   

\begin{figure}[ht]
\centering
\includegraphics[width=0.9\columnwidth]{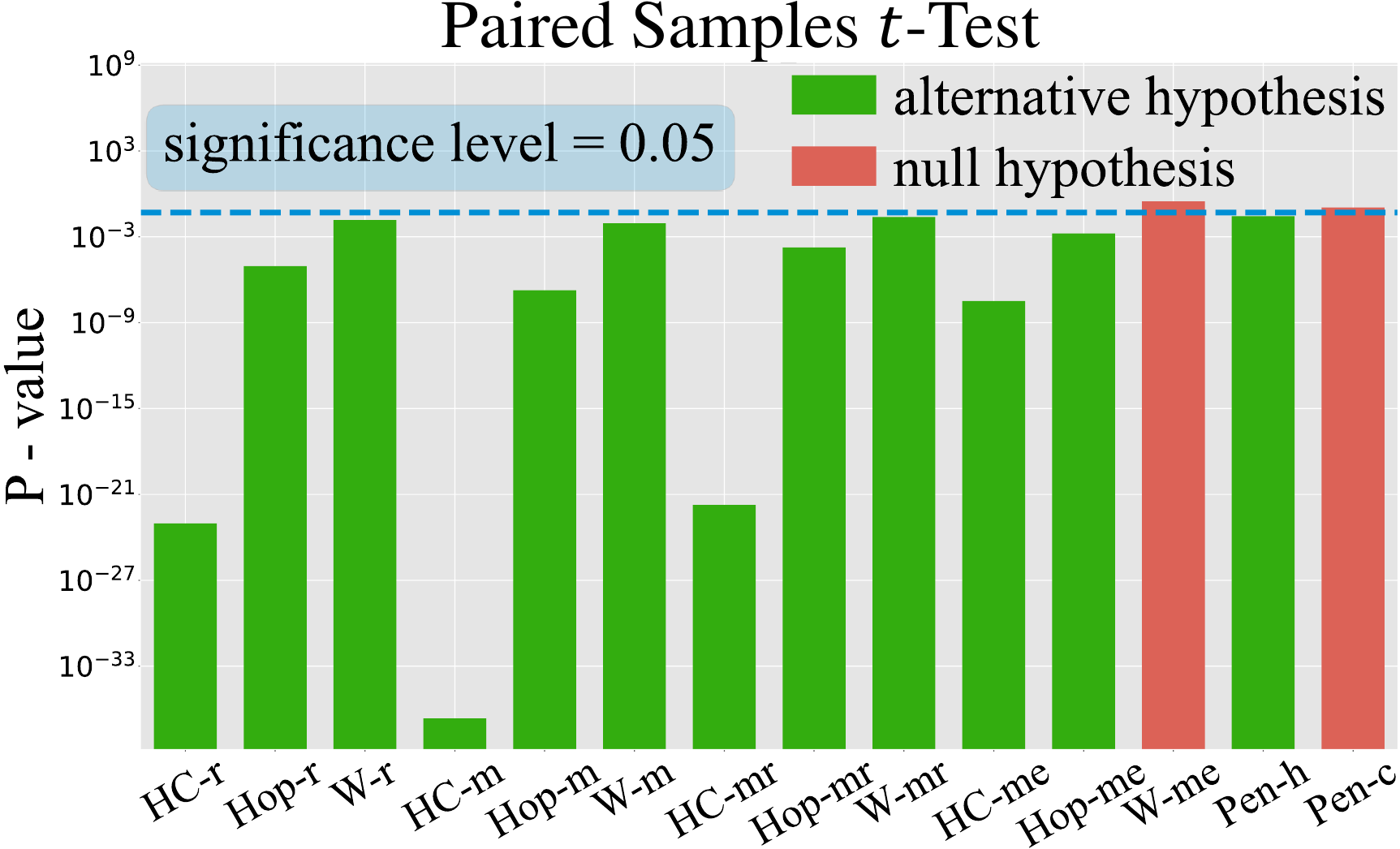}
\vspace{-8pt}
\caption{ Statistical validations.
The alternative hypothesis means GORL significantly outperforms the vanilla algorithms while the null one stands for insignificant~differences. Results show that GORL brings statistically significant performance improvements on most of the tasks under the 0.05 significant level. (HC = Halfcheetah, Hop = Hopper, W = Walker, r = random, m = medium, mr = medium-replay, me = medium-expert, e = expert, h = human, c = cloned.)}
\label{fig:t_test}
\vspace{-10pt}
\end{figure}

\textbf{Paired samples \textit{t}-test.}
Considering that offline-trained policies usually perform with high variance, we additionally conduct the paired-samples $t$-test validation.
The purpose of this experiment is to determine whether there is statistical evidence that the mean difference between the baseline method and the proposed method is significantly different from zero.
50 samples from the last 10 evaluations and 5 seeds for each task are paired in experiments and the significance level is set as $\alpha=0.05$.
If the p-value is above the significance level, the null hypothesis (i.e., the mean difference is not significantly different from zero) is accepted. Otherwise, the alternative hypothesis is accepted.
Figure 2 shows a general acceptance of the alternative hypothesis, which means that our GORL indeed brings a statistically significant improvement to the TD3+BC algorithm.
Furthermore, Table~\ref{t_test_all} provides a more comprehensive validation of all four base algorithms on both locomotion and adroit datasets.
On average, GORL consistently brings significant performance improvement to all of the base algorithms.

\textbf{Policy constraint.} To analyze GORL's effects on the policy constraint intensity, we record the policy constraint loss of TD3+BC~\cite{td3+bc} and TD3+BC with GORL during the whole training process, as shown in Figure~\ref{fig:nomalized policy constraint terms}.
It can be observed that GORL significantly increases the relative policy constraint intensity of the expert data.
The observation reveals that GORL can indeed encourage the algorithms to stress more on the policy constraint for high-quality samples.

\begin{figure}[ht]
    \centering
    \includegraphics[width=1.0\columnwidth]{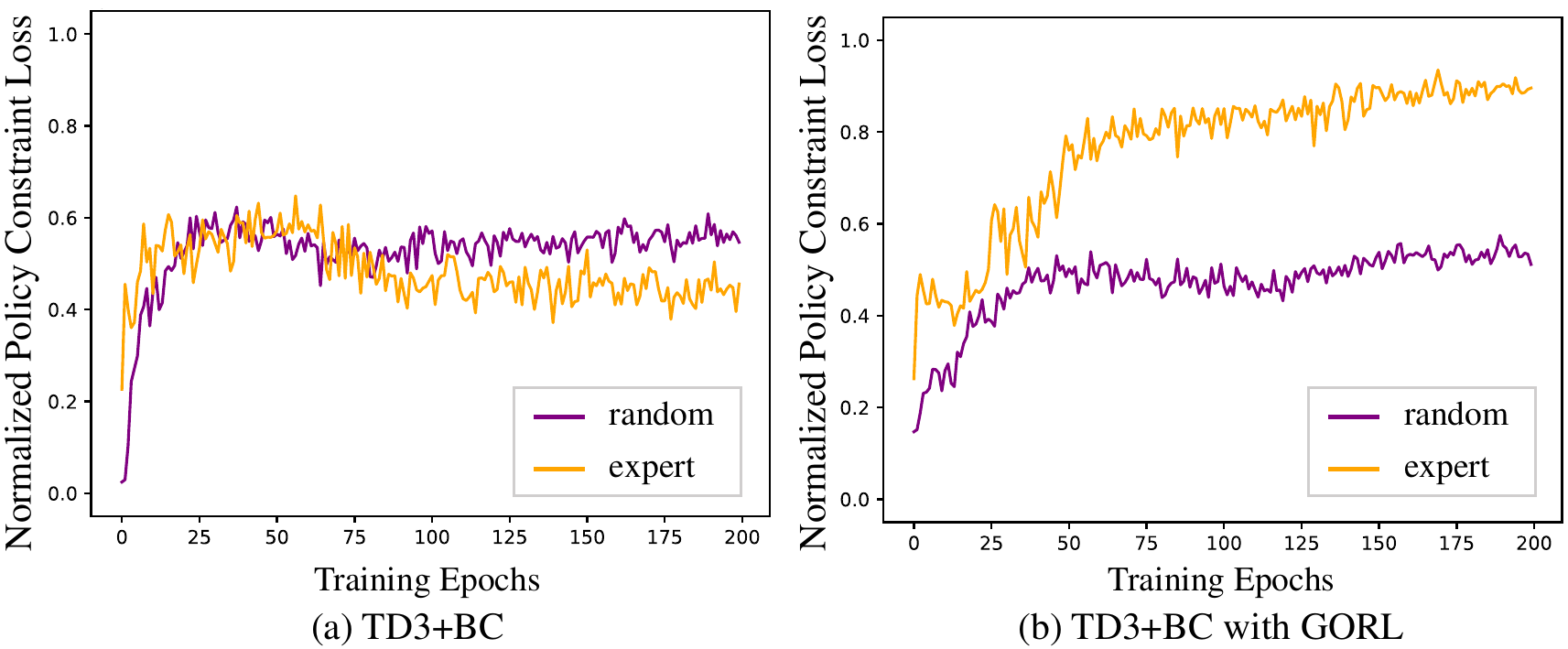}
    \vspace{-18pt}
    \caption{The policy constraint loss of TD3+BC and TD3+BC with GORL during the training process. For a quick comparison, we normalize the curves in each figure to range $[0, 1]$. GORL can greatly increase the relative policy constraint intensity of the expert data.}
    \label{fig:nomalized policy constraint terms}
\end{figure}

\textbf{Runtime.} 
As shown in Table~\ref{run_time}, we measure the training wall-clock time of algorithms w/o and w/ GORL on D4RL locomotion tasks.
On average, GORL only brings $1.9\%$ extra training cost to the base algorithms.

%% file: chapters/5_discussion.tex
\subsection{Are adaptive weights better than the fixed weight?}
During guided learning, each sample is assigned a different weight and the weights vary through training.
Specifically, when fed with relatively high-quality samples, the agent may be inclined to imitation learning;
otherwise, when encountering lower-quality samples, it may choose to slightly diverge from these samples' distribution.
We claim that such adaptive weights seek to achieve the full potential of every sample, leading to higher performance compared with the fixed weight.
To verify this argument, we compare the fixed-weight method with our adaptive-weights approach on TD3+BC~\cite{td3+bc} algorithm. The possible values of the fixed weight are in the set $\{0.0, 0.1, 0.3, 0.5, 0.7, 0.9, 1.0\}$, and the range of adaptive weights is $[0.0, 1.0]$. 
As shown by the total scores in Table~\ref{fixweight}, our adaptive-weights approach outperform all the fixed-weight methods by a large margin ($\ge 52.2$).
Meanwhile, since the policy's performance varies greatly with different values of the fixed weight, the fixed-weight methods might require much more parameter tuning than our adaptive-weights approach.

\begin{table}[t]
  \caption{Results of paired samples \textit{t}-test, which statistically validate the significance ($\alpha<0.05$) of performance improvement brought by GORL averaged on all datasets.}
  \centering
    \begin{tabular}{p{15mm}<{\centering} p{10.5mm}<{\centering} p{10.5mm}<{\centering} p{10.5mm}<{\centering} p{10.5mm}<{\centering}}
    \hline
	& TD3+BC
	& SAC+BC 
	& CQL
	& IQL
 \\ \hline
    P-Value\ \ ($\alpha$)
    & 9.33E-08 & 8.96E-15 & 3.77E-09 & 0.023349  \\ 
    Significance 
    & TRUE & TRUE & TRUE & TRUE  \\ 
    \hline
  \end{tabular}
  \label{t_test_all}
  \vspace{-8pt}
\end{table}

\begin{table}[t]
  \small
  \caption{Training wall-clock time on GPU 2080Ti.}
  \centering
    \begin{tabular}{p{26mm} p{13mm}<{\centering} p{13mm}<{\centering} p{13mm}<{\centering}}
    \toprule
    Runtime
	&  Base
	&  Ours
 	& Increase
	\\ \midrule
    TD3+BC~\cite{td3+bc}
    & 2h 12m & 2h 17m & 3.7\%  \\ 
    SAC+BC~\cite{td3+bc,sac} 
    & 2h 56m & 2h 59m & 1.7\%  \\ 
    IQL~\cite{iql} 
    & 3h 52m & 3h 58m & 2.6\%  \\ 
    CQL~\cite{cql} 
    & 9h 12m & 9h 19m & 1.3\%  \\ 
    \midrule
    Average 
    & 4h 33m & 4h 38m & 1.9\%  \\
    \bottomrule
  \end{tabular}
  \label{run_time}
\end{table}

\begin{table*}
  \caption{Comparison between fixed weights and adaptive weights (ours) based on TD3+BC~\cite{td3+bc}.}
  \centering
    \begin{tabular}{p{39.5mm} p{9mm}<{\centering} p{9mm}<{\centering} p{9mm}<{\centering} p{9mm}<{\centering} p{9mm}<{\centering} p{9mm}<{\centering} p{9mm}<{\centering} p{11mm}<{\centering}}
    \hline
	\multirow{2}{*}{Dataset}
	& \multicolumn{7}{c}{Fixed Weight} 
    & Ours
	\\ 
	\cmidrule(lr){2-8}\cmidrule(lr){9-9}
	& 0.0
	& 0.1
	& 0.3
	& 0.5
	& 0.7
	& 0.9
	& 1.0
	& 0$\sim$1
	\\ \hline
    halfcheetah-random-v2 
    & \textbf{32.5} 
    & 24.9 & 19.3 & 14.6
    & 13.5 & 12.2 & 11.2
    & 16.5 \\   
    hopper-random-v2 
    & 15.5 
    & \textbf{24.7} & 16.6 & 8.3
    & 8.3 & 8.6 & 8.8
    & 16.8 \\   
    walker2d-random-v2 
    & 0.1 
    & -0.2 & 0.1 & -0.2
    & 0.4 & 1.6 & 1.3
    & \textbf{2.3} \\  
    halfcheetah-medium-v2 
    & 20.2 
    & \textbf{59.8} & 53.7 & 50.7
    & 49.4 & 48.7 & 48.2
    & 51.8 \\    
    hopper-medium-v2 
    & 0.6 
    & 41.0 & \textbf{81.0} & 64.6
    & 61.0 & 59.8 & 59.8
    & 67.1 \\    
    walker2d-medium-v2 
    & 0.9 
    & 1.9 & 54.0 & 85.2
    & 84.8 & 83.5 & 83.9
    & \textbf{85.7} \\   
    halfcheetah-medium-replay-v2 
    & 44.5 
    & \textbf{52.2} & 48.3 & 46.3
    & 45.5 & 44.9 & 44.6
    & 46.7 \\  
    hopper-medium-replay-v2 
    & 33.7 
    & \textbf{86.9} & 85.7 & 64.5
    & 68.2 & 69.9 & 61.7
    & 74.7 \\   
    walker2d-medium-replay-v2 
    & 13.2 
    & 22.4 & \textbf{86.3} & 83.1
    & 79.8 & 77.1 & 79.5
    & 84.5 \\   
    halfcheetah-medium-expert-v2 
    & 6.5 
    & 44.9 & 71.0 & 75.7
    & 82.1 & 89.1 & 91.8
    & \textbf{95.6} \\    
    hopper-medium-expert-v2 
    & 0.9 
    & 3.2 & 80.1 & 100.3
    & 99.6 & 98.6 & 98.6
    & \textbf{105.3} \\   
    walker2d-medium-expert-v2 
    & -0.3 
    & 4.4 & 68.2 & 108.1
    & 108.1 & \textbf{110.2} & \textbf{110.2}
    & 109.5 \\  
    \hline
    total 
    & 168.5 
    & 366.1 & 664.3 & 701.1
    & 700.7 & 704.3 & 699.5
    & \textbf{756.5} \\   
    \hline
  \end{tabular}
  \label{fixweight}
\end{table*}

\begin{figure*}
\centering
    \includegraphics[width=0.98\textwidth]{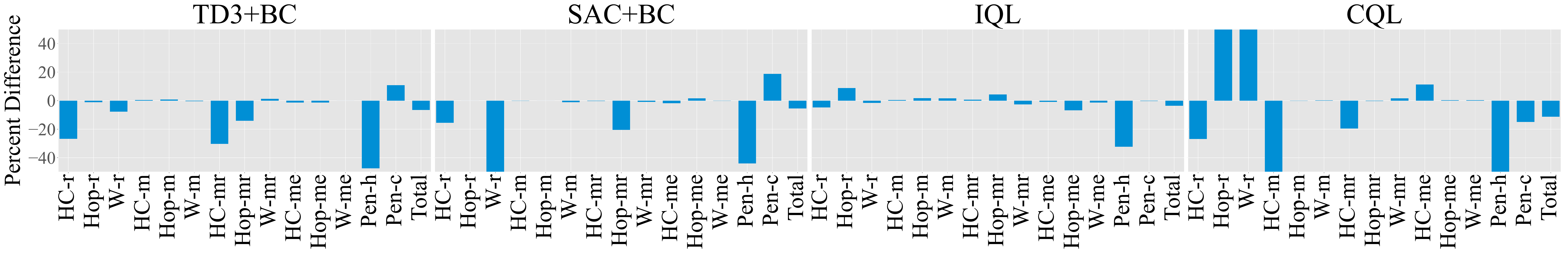}
    \vspace{-8pt}
    \caption{Percent performance difference of vanilla training \textbf{with mixed data} compared to that \textbf{with pure~offline data}. Mixed data means adding 200 expert samples to the offline data without guidance. General performance drops show  simply mixing limited expert samples with offline data is usually harmful.}
    \label{mixbaseline}
\end{figure*}

\begin{table}[ht]
  \small
  \caption{Comparison between \textbf{action selection} and \textbf{guided training}. CQL$^{(G)}$ represents the CQL~\cite{cql} with our GORL.}
  \centering
    \begin{tabular}{p{10.5mm} p{17mm}<{\centering} p{11mm}<{\centering} p{15mm}<{\centering} p{12mm}<{\centering}}
    \hline
	Dataset
	& Filt. BC~\cite{decision_transformer}
	& DT~\cite{decision_transformer}
	& RvS-R~\cite{RVS}
	& CQL$^{(G)}$
	\\ \hline
    HC-r
    & 2.0 & 2.2 & 3.9 & \textbf{20.3} \\ 
    Hop-r 
    & 4.1 & 7.5 & 7.7 & \textbf{9.8} \\ 
    W-r 
    & 1.7 & 2.0 & -0.2 & \textbf{3.4} \\ 
    HC-m 
    & 42.5 & 42.6 & 41.6 & \textbf{51.3} \\ 
    Hop-m 
    & 56.9 & 67.6 & 60.2 & \textbf{72.2} \\ 
    W-m 
    & 75.0 & 74.0 & 71.7 & \textbf{83.5} \\
    HC-mr
    & 40.6 & 36.6 & 38.0 & \textbf{47.0} \\
    Hop-mr
    & 75.9 & 82.7 & 73.5 & \textbf{94.6} \\
    W-mr
    & 62.5 & 66.6 & 60.6 & \textbf{80.4} \\ 
    HC-me
    & \textbf{92.9} & 86.8 & 92.2 & 81.1 \\ 
    Hop-me
    & \textbf{110.9} & 107.6 & 101.7 & 105.3 \\
    W-me
    & 109.0 & 108.1 & 106.0 & \textbf{109.2} \\ 
    \hline
    total 
    & 674.0 & 684.3 & 656.9 & \textbf{758.1} \\
    \hline
  \end{tabular}
  \label{action_selection}
\end{table}

\subsection{Does limited expert data benefit vanilla training?}

Acquiring a large amount of expert data may be very expensive or even impossible in the real world.
Hence, it is of great value to make the best use of limited high-quality data.
Although the proposed guided training achieves significant improvement on existing offline RL algorithms, it remains unclear whether the benefit comes from the training scheme or the extra expert data.
In other words, would the limited expert data also work well in the vanilla training? 
To investigate that, we obtain the mixed offline data by simply adding the limited expert data into the pure offline data and randomly mixing them.
Following the vanilla training ways of baselines, we compare the percent difference of the performance with mixed data to that with pure offline data, as shown in Figure~\ref{mixbaseline}.
Experiments validate that the vanilla training methods cannot benefit from the limited high-quality data.
On the contrary, the distributional gap between the pure offline data and the limited expert data makes the optimization process harder and more unstable, leading to significant performance degradation.

\subsection{How does GORL differ from action selection?}

The proposed GORL learns an adaptive weight for every individual sample, while action selection aims to pick the best samples for offline learning.
On the one hand, both guided training and action selection can be considered as conditional behavior cloning, which adjusts the degree of policy constraint according to specific conditions. 
On the other hand, guided training can be more informative and efficient since it makes better use of the entire dataset, rather than only the high-performing samples.
We compare to the Filtering (``Filt.'') BC~\cite{decision_transformer}, DT~\cite{decision_transformer}, and RvS-R~\cite{RVS} algorithms on the locomotion datatsets, as shown in Table~\ref{action_selection}.
DT adopts transformer architectures to perform behavior cloning on a subset of the data. Filt.~BC performs BC after filtering for the 10\% trajectories with the highest returns. RvS-R uses supervised learning to condition on reward values during reinforcement learning. The baselines' results are borrowed from their original papers~\cite{decision_transformer,RVS}. 
Comparison results indicate that our guided training greatly outperforms the  action selection methods.

\subsection{What if we possess more expert data?}   \label{sec: different num of expert data}

Compared with the vanilla algorithm which simply mixes expert demonstrations with the offline dataset, the guided training better utilizes the limited high-quality data.
However, it is also notable that the offline data which is mixed with abundant expert samples can also result in superior performances under the vanilla algorithm.
For example, as shown in Table~\ref{main_results}, the overall performance of baseline algorithms on the medium-expert dataset is much better than that on the medium dataset. 
Hence, we further conduct experiments on different numbers of expert samples and draw some empirical conclusions on the best way of using expert data.
In Figure~\ref{expertsize}a, the vanilla scheme with expert-mixed data (denoted as ``D(e) $+$ D'') and the guided scheme (denoted as ``D(e) $\xrightarrow{}$ D'') are compared to the baseline scheme with pure offline data.
When the amount of expert data is small, the guided scheme constantly outperforms the mixed scheme.
However, when the amount of expert samples reaches $10^5$, the mixed scheme gains greater benefits.
Furthermore, we train the policy on the expert-only dataset (denoted as ``D(e)'') with different dataset scales, as shown in Figure~\ref{expertsize}b.
It's obvious that the policy's scores remain quite low until the expert sample number reaches $10^4$, which coincides with our suggestion in Section~\ref{sec: intuitions behind GORL} that a large amount of training data is necessary for offline RL.
In conclusion: (1)~the limited expert data itself cannot produce a satisfying agent, due to the insufficiency of training samples; 
(2)~GORL can generate reliable guidance for offline RL with only a few expert samples (\textit{e.g.}, 100), but performance improvement would become insignificant if expert demonstrations increase excessively.

\begin{figure}[t]
\centering
\includegraphics[width=0.4\textwidth]{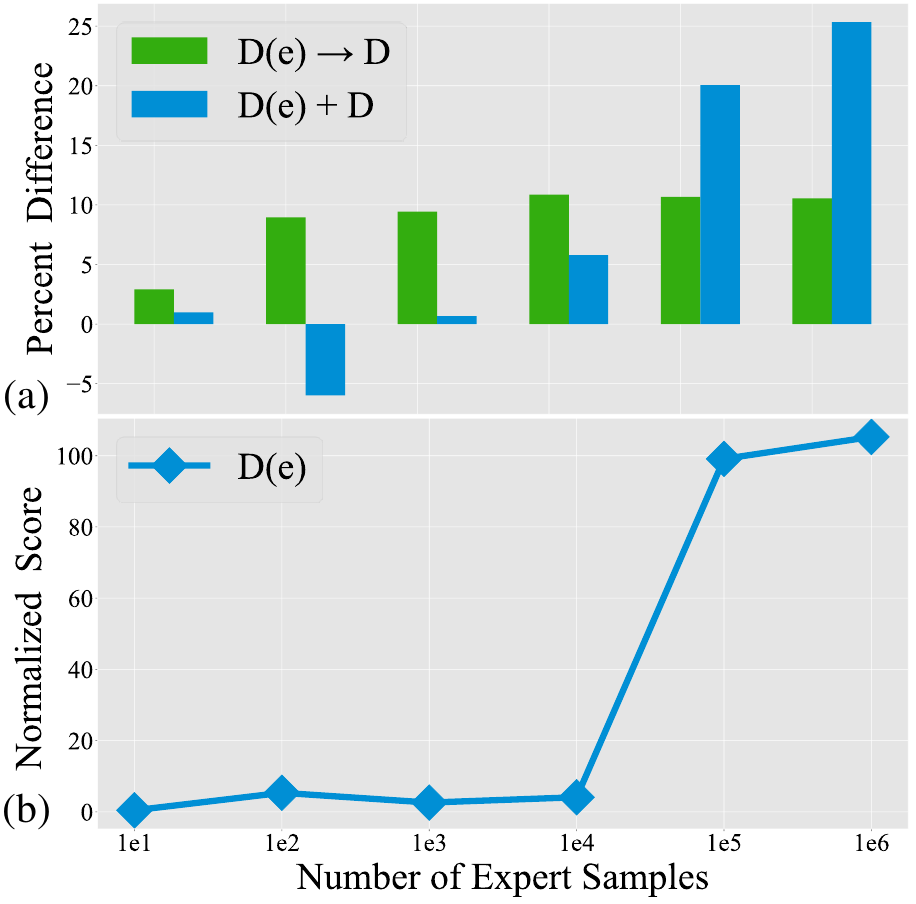}
\caption{Results from various numbers of expert samples. \textbf{D} represents the abundant low-quality data and \textbf{D(e)} refers to the expert samples. \textbf{(a)} \textbf{``D(e) $\xrightarrow{}$ D''} means training a policy on D with the guidance of D(e). \textbf{``D(e) $+$ D''} means training a policy on the vanilla mixture of D and D(e). A larger positive percent difference shows a larger performance improvement, compared to the policy simply trained from D. \textbf{(b)} \textbf{``D(e)''} means training a policy on D(e) without D. A larger score (normalized according to D4RL~\cite{Fu2020D4RLDF}) corresponds to a better policy.}
\label{expertsize}
\end{figure}

%% file: chapters/6_related_work.tex
\textbf{Offline RL.}~Due to the state-action distribution gap between the training dataset and the test environment, offline RL suffers from the \textit{distributional shift}~\cite{offline_survey_sergey,9904958}~problem.
Much prior work attempts to mitigate this issue through constraining or regularizing the learned policy to be approximated to the behavioral policy. As mentioned in BRAC~\cite{brac}, both explicit and implicit approaches are~beneficial. 

Some methods explicitly implement the constraint by adding a policy constraint term to the policy improvement equation~\cite{td3+bc,PPO+BC,Accelerating_Online_Reinforcement_Learning_with_Offline_Datasets,Overcoming_Exploration_in_Reinforcement_Learning_with_Demonstrations,bear}.
Specifically, inspired by PPO+BC~\cite{PPO+BC}, DDPG+BC~\cite{Overcoming_Exploration_in_Reinforcement_Learning_with_Demonstrations}, and AWAC~\cite{Accelerating_Online_Reinforcement_Learning_with_Offline_Datasets}, TD3+BC~\cite{td3+bc} converts the online RL algorithm (\textit{i.e.} TD3~\cite{td3}) into the offline form via adding a term of behavior cloning loss to the actor's policy improvement loss.
Besides, BEAR~\cite{bear} utilizes a distribution-constrained policy to lessen the accumulation of bootstrapping errors.

Other approaches confine the policy implicitly by revising update rules of value functions~\cite{AlgaeDICE,cql,iql,The-Importance-of-Pessimism-in-Fixed-Dataset-Policy-Optimization}.
In more detail, by including an extra CQL regularization term in the Q-function update equation, Kumar \textit{et al.}~\cite{cql} propose a method that learns a conservative Q-function to estimate a lower bound of the value function.
Instead of evaluating out-of-distribution actions directly, IQL~\cite{iql} only estimates the maximum Q-value over actions within the data distribution via utilizing expectile regression.

There are also methods leveraging imitation learning to assist RL training, called \textit{RvS} (offline RL via supervised learning)~\cite{RVS}, which are commonly imitation learning methods conditioned on goals~\cite{End-to-End_Driving_Via_Conditional_Imitation_Learning,Learning_to_Reach_Goals_via_Iterated_Supervised_Learning,Learning-Latent-Plans-from-Play} or reward values~\cite{Reward-Conditioned_Policies,decision_transformer,Training_Agents_using_Upside-Down_Reinforcement_Learning}. 
Reweighting or filtering are also adopted to advantage high-performing actions~\cite{Rewriting_History_with_Inverse_RL,Fitted_Q-iteration,Reward-Conditioned_Policies,Advantage-Weighted_Regression,Exponentially_Weighted_Imitation_Learning_for_Batched_Historical_Data,BAIL}.

\textbf{Sample weighting method.}~To mitigate the problem of overfitting bias in training data, prior methods attempt to design a weighting function that  maps from training loss to sample weight.
Researchers start by designing such functions manually.
They focus on pre-designing a weighting function generating sample weights with training losses as input.
Generally speaking, methods in this category either force sample weights to monotonically increase (such as AdaBoost~\cite{Adaboost}, hard example mining~\cite{Ensemble_of_exemplar-SVMs_for_object_detection_and_beyond}, and focal loss~\cite{focal_loss}) or decrease (such as SPL~\cite{SPL}, Active bias~\cite{Active_Bias} and iterative reweighting~\cite{A_Framework_for_Robust_Subspace_Learning}). 

Inspired by previous meta-learning approaches~\cite{maml,relation_network,Meta_Networks,LSTM}, some  methods learn adaptive weighting functions via meta-learning. 
FWL~\cite{Fidelity-Weighted_Learning} proposes a semi-supervised algorithm leveraging only a small quantity of high-quality data and  a large set of unlabeled data samples.
Learning to Teach~\cite{Learning_to_Teach, Learning_to_Teach_with_Dynamic_Loss_Functions} adopts a reinforcement learning agent as the teacher model to facilitate the training of the student model.
For a similar purpose,  MentorNet~\cite{MentorNet} leverages a bidirectional LSTM network~\cite{LSTM} to supervise the training of StudentNet. 
With the guidance of a small number of unbiased meta-data, Meta-Weight-Net~\cite{meta-weight-net} also utilizes an MLP with one hidden layer to mitigate the problem of overfitting. 
Different from the methods mentioned above, L2RW~\cite{Learning_to_Reweight_Examples_for_Robust_Deep_Learning} learns weights from gradient directions, without an explicit network.

%% file: chapters/7_conclusion.tex
We present Guided Offline RL (GORL), a general training framework compatible with most offline RL algorithms, to learn an adaptive intensity of policy constraint under the guidance of only \textit{a few} high-quality data.
Specifically, GORL exerts a weak (or strong) constraint to ``random-like'' (or ``expert-like'')  samples in the offline dataset.
To our knowledge, this is the first offline RL method that takes full advantage of a small number of expert demonstrations. 
Theoretically, we prove that even \textit{quite limited} expert data can provide reliable guidance.
Empirically, we validate our method on several popular offline RL algorithms and abundant tasks in the D4RL benchmark.
Moreover, we discuss the benefits of adaptive policy constraints and guiding expert data through various ablation studies.
Nevertheless, although we investigate the guidance of a few expert data, the feasibility of achieving the adaptive constraint without any guiding data remains to be explored in future work.

%% file: chapters/8_appendixA.tex
\subsection{Proof of Theorem 1} \label{appendix: proof of the intuition of GORL}

Denote 
$L^{\text{policy}}_{k_1}(\boldsymbol{\theta}) =  L_{pc}(a_{k_1}, \pi_{\boldsymbol{\theta}}(s_{k_1}))$ with $(s_{k_1}, a_{k_1})\sim \mathcal{D}$,
and 
$L^{\text{guide}}(\boldsymbol{\theta}) = 
\frac{1}{n_{\mathcal{G}}} \sum_{{k_2}=1}^{n_{\mathcal{G}}} L_{k_2}^{\text{guide}}(\boldsymbol{\theta})
= \frac{1}{n_{\mathcal{G}}}\sum_{{k_2}=1}^{n_{\mathcal{G}}}\left[L_{pc}(a_{k_2}, \pi_{\boldsymbol{\theta}}(s_{k_2}))\right]$.
\begin{equation}
\begin{aligned}
    &\left.\frac{\partial L^{\text{guide}}(\hat{\boldsymbol{\theta}}^{(t)}(\boldsymbol{w}))}{\partial \boldsymbol{w}}\right|_{\boldsymbol{w}^{(t)}}\\[4mm]
    & = 
    \left( \left. \frac{\partial \hat{\boldsymbol{\theta}}^{(t)}(\boldsymbol{w})}{\partial \boldsymbol{w}} \right|_{\boldsymbol{w}^{(t)}} \right)^{\top}
    \left.\frac{\partial L^{\text{guide}}(\boldsymbol{\theta})}{\partial \boldsymbol{\theta}}\right|_{\hat{\boldsymbol{\theta}}^{(t)}(\boldsymbol{w}^{(t)})} 
    \\[4mm]
    & = \Bigg( -\frac{\alpha_{\mathcal{D}}}{n_{\mathcal{D}}}\sum_{{k_1}=1}^{n_{\mathcal{D}}}
    \left.\frac{\partial L_{pc}(a_{k_1}, \pi_{\boldsymbol{\theta}}(s_{k_1}))}{\partial \boldsymbol{\theta}}\right|_{\boldsymbol{\theta}^{(t)}}\\[4mm]
    & \qquad \cdot \left(\left.\frac{\partial \mathcal{B}_{\boldsymbol{w}}(L_{pc}(a_{k_1}, \pi_{\boldsymbol{\theta^{(t)}}}(s_{k_1})))}{\partial \boldsymbol{w}}\right|_{\boldsymbol{w}^{(t)}}\right)^{\top}
    \Bigg)^{\top} \\[4mm]
    &\qquad \cdot \left.\frac{\partial L^{\text{guide}}(\boldsymbol{\theta})}{\partial \boldsymbol{\theta}}\right|_{\hat{\boldsymbol{\theta}}^{(t)}(\boldsymbol{w}^{(t)})} \\[4mm]
    & = -\frac{\alpha_{\mathcal{D}}}{n_{\mathcal{D}}}
    \sum_{{k_1}=1}^{n_{\mathcal{D}}}
    \Bigg[\left.\frac{\partial \mathcal{B}_{\boldsymbol{w}}(L_{pc}(a_{k_1}, \pi_{\boldsymbol{\theta^{(t)}}}(s_{k_1})))}{\partial \boldsymbol{w}}\right|_{\boldsymbol{w}^{(t)}} \\[4mm]
    & \qquad \cdot \left(\left.\frac{\partial L_{pc}(a_{k_1}, \pi_{\boldsymbol{\theta}(s_{k_1})})}{\partial \boldsymbol{\theta}}\right|_{\boldsymbol{\theta}^{(t)}}\right)^{\top}
    \left.\frac{\partial L^{\text{guide}}(\boldsymbol{\theta})}{\partial \boldsymbol{\theta}}\right|_{\hat{\boldsymbol{\theta}}^{(t)}(\boldsymbol{w}^{(t)})} \Bigg]\\[4mm]
    & =-\frac{\alpha_{\mathcal{D}}}{n_{\mathcal{D}}}
    \sum_{{k_1}=1}^{n_{\mathcal{D}}}
    \Bigg[\left.\frac{\partial \mathcal{B}_{\boldsymbol{w}}(L_{pc}(a_{k_1}, \pi_{\boldsymbol{\theta^{(t)}}}(s_{k_1})))}{\partial \boldsymbol{w}}\right|_{\boldsymbol{w}^{(t)}} \\[4mm]
    & \qquad \cdot \left(\left.\frac{\partial L_{pc}(a_{k_1}, \pi_{\boldsymbol{\theta}}(s_{k_1}))}{\partial \boldsymbol{\theta}}\right|_{\boldsymbol{\theta}^{(t)}}\right)^{\top} \\[4mm]
    & \qquad \cdot  \frac{1}{n_{\mathcal{G}}}\sum_{k_2=1}^{n_{\mathcal{G}}}\left.\frac{\partial L^{\text{guide}}_{k_2}(\boldsymbol{\theta})}{\partial \boldsymbol{\theta}}\right|_{\hat{\boldsymbol{\theta}}^{(t)}(\boldsymbol{w}^{(t)})} \Bigg]\\[4mm]
    & = -\frac{\alpha_{\mathcal{D}}}{n_{\mathcal{D}}}
    \sum_{{k_1}=1}^{n_{\mathcal{D}}} 
    \Bigg[\left( \frac{1}{n_{\mathcal{G}}}\sum_{k_2=1}^{n_{\mathcal{G}}}\left.\frac{\partial L^{\text{guide}}_{k_2}(\boldsymbol{\theta})}{\partial \boldsymbol{\theta}}\right|_{\hat{\boldsymbol{\theta}}^{(t)}(\boldsymbol{w}^{(t)})} \right)^{\top} \\[4mm]
    & \qquad \cdot \left.\frac{\partial L^{\text{policy}}_{k_1}(\boldsymbol{\theta})}{\partial \boldsymbol{\theta}}\right|_{\boldsymbol{\theta}^{(t)}}
     \left.\frac{\partial \mathcal{B}_{\boldsymbol{w}}(L^{\text{policy}}_{k_1}(\boldsymbol{\theta}^{(t)}))}{\partial \boldsymbol{w}}\right|_{\boldsymbol{w}^{(t)}} \Bigg] \\[4mm]
     & = -\frac{\alpha_{\mathcal{D}}}{n_{\mathcal{D}}}
    \sum_{{k_1}=1}^{n_{\mathcal{D}}} C_{k_1} \left.\frac{\partial \mathcal{B}_{\boldsymbol{w}}(L^{\text{policy}}_{k_1}(\boldsymbol{\theta}^{(t)}))}{\partial \boldsymbol{w}}\right|_{\boldsymbol{w}^{(t)}},
\end{aligned}
\nonumber
\end{equation}
where 
\begin{equation}
    C_{k_1} =  \left( \frac{1}{n_{\mathcal{G}}}\sum_{k_2=1}^{n_{\mathcal{G}}}\left.\frac{\partial L^{\text{guide}}_{k_2}(\boldsymbol{\theta})}{\partial \boldsymbol{\theta}}\right|_{\hat{\boldsymbol{\theta}}^{(t)}(\boldsymbol{w}^{(t)})} \right)^{\top} 
     \left.\frac{\partial L^{\text{policy}}_{k_1}(\boldsymbol{\theta})}{\partial \boldsymbol{\theta}}\right|_{\boldsymbol{\theta}^{(t)}}. \nonumber
\end{equation}
     
Based on the observation above, Equation (4) can be rewritten as:
\begin{equation}
\begin{aligned}
    \boldsymbol{w}^{(t+1)} & = \boldsymbol{w}^{(t)} - \alpha_{\mathcal{G}} \left.\frac{\partial L^{\text{guide}}(\hat{\boldsymbol{\theta}}^{(t)}(\boldsymbol{w}))}{\partial \boldsymbol{w}}\right|_{\boldsymbol{w}^{(t)}} \\[4mm]
    & = \boldsymbol{w}^{(t)} + \frac{\alpha_{\mathcal{D}} \alpha_{\mathcal{G}}}{n_{\mathcal{D}}}
    \sum_{{k_1}=1}^{n_{\mathcal{D}}} C_{k_1} \left.\frac{\partial \mathcal{B}_{\boldsymbol{w}}(L^{\text{policy}}_{k_1}(\boldsymbol{\theta}^{(t)}))}{\partial \boldsymbol{w}}\right|_{\boldsymbol{w}^{(t)}}. \nonumber
\end{aligned}
\end{equation}
Therefore, Theorem 1 is proven.

\subsection{Proof of Theorem 2} \label{appendix: proof of the guiding gradient gap}
We start proving Theorem 2 by Lemma~\ref{lemma: gap between meta and optimal gradient} derived from Kolmogorov's inequality.
\begin{lemma} \label{lemma: gap between meta and optimal gradient}
For independent random variables $\zeta_1, \zeta_2, \cdots, \zeta_n$, supposing that $\mathbb{E}\zeta_1^2, \mathbb{E} \zeta_2^2, \cdots, \mathbb{E} \zeta_n^2$ exist, then the following inequality holds:
\begin{equation}
     \forall \epsilon > 0, \quad P\left(\left|\sum_{i=1}^{n} \zeta_{i}\right| \geq \epsilon\right) \leq \frac{1}{\epsilon^{2}} \sum_{i=1}^{n} \mathbb{E} \zeta_{i}^{2}.
    \label{eq: eq1 of lemma gap between meta and optimal gradient}
\end{equation}
Furthermore, if $\mathbb{E}\zeta_i = 0 \ (i=1, 2, \cdots, n)$, then
\begin{equation}
    \forall \epsilon > 0, \quad P\left(\left|\sum_{i=1}^{n} \zeta_{i}\right| \geq \epsilon\right) \leq \frac{1}{\epsilon^{2}} \sum_{i=1}^{n} \Var (\zeta_{i}).
    \label{eq: eq2 of lemma gap between meta and optimal gradient}
\end{equation}
\end{lemma}
\begin{proof}
By Kolmogorov's inequality, we know that 
\begin{equation}
    \forall \epsilon > 0, \quad P\left(\max _{1 \leq k \leq n}\left|\sum_{i=1}^{k} \zeta_{i}\right| \geq \epsilon\right) \leq \frac{1}{\epsilon^{2}} \sum_{i=1}^{n} \mathbb{E} \zeta_{i}^{2}.  \nonumber 
\end{equation}
Noticing that 
\begin{equation}
 P\left(\left|\sum_{i=1}^{n} \zeta_{i}\right| \geq \epsilon\right) \le P\left(\max _{1 \leq k \leq n}\left|\sum_{i=1}^{k} \zeta_{i}\right| \geq \epsilon\right),   
\end{equation}
Equation~\eqref{eq: eq1 of lemma gap between meta and optimal gradient} is proven.
Moreover, when $\mathbb{E}\zeta_i = 0$, we can infer that $\mathbb{E}\zeta_i^2 = \Var(\zeta_i) + (\mathbb{E}\zeta_i)^2 = \Var(\zeta_i)$. Substituting $\mathbb{E}\zeta_i^2$ in Equation~\eqref{eq: eq1 of lemma gap between meta and optimal gradient} with $\Var(\zeta_i)$, Equation~\eqref{eq: eq2 of lemma gap between meta and optimal gradient} is proven.
\end{proof}
We continue to prove Theorem 2 by noticing that:
\begin{equation}
\begin{aligned}
    \Bigg\| \frac{\partial L_{1:n}^{\operatorname{guide}}(\hat{\boldsymbol{\theta}})}{\partial \hat{\boldsymbol{\theta}}^{[l]}} -& \frac{\partial L^{\operatorname{guide}}_*}{\partial \hat{\boldsymbol{\theta}}^{[l]}} \Bigg\|_1
    = \left\| \frac{1}{n} \sum_{k=1}^n\frac{\partial L_{k}^{\operatorname{guide}}(\hat{\boldsymbol{\theta}})}{\partial \hat{\boldsymbol{\theta}}^{[l]}} - \frac{\partial L^{\operatorname{guide}}_*}{\partial \hat{\boldsymbol{\theta}}^{[l]}} \right\|_1\\[4mm]
    &= \frac{1}{n}  \left\| \sum_{k=1}^n\left( \frac{\partial L_{k}^{\operatorname{guide}}(\hat{\boldsymbol{\theta}})}{\partial \hat{\boldsymbol{\theta}}^{[l]}} -  \frac{\partial L^{\operatorname{guide}}_*}{\partial \hat{\boldsymbol{\theta}}^{[l]}} \right) \right\|_1\\[4mm]
    &\le \frac{1}{n} \sum_{k=1}^n \left\| \frac{\partial L_{k}^{\operatorname{guide}}(\hat{\boldsymbol{\theta}})}{\partial \hat{\boldsymbol{\theta}}^{[l]}} - \frac{\partial L^{\operatorname{guide}}_*}{\partial \hat{\boldsymbol{\theta}}^{[l]}} \right\|_1\\[4mm]
    &=  \frac{1}{n} \sum_{k=1}^n  \sum_{i=1}^{d_1} \sum_{j=1}^{d_2} \left| \frac{\partial L_{k}^{\operatorname{guide}}(\hat{\boldsymbol{\theta}})}{\partial \hat{\boldsymbol{\theta}}_{ij}^{[l]}} - \frac{\partial L^{\operatorname{guide}}_*}{\partial \hat{\boldsymbol{\theta}}_{ij}^{[l]}} \right|. \nonumber 
\end{aligned}
\end{equation}
According to  Theorem 2's assumption, $\frac{\partial L_{k}^{\operatorname{guide}}(\hat{\boldsymbol{\theta}})}{\partial \hat{\boldsymbol{\theta}}_{ij}^{[l]}} \ (k\in \{1, 2, \cdots, n\}, i \in \{1, 2, \cdots, d_1\}, j \in \{1, 2, \cdots, d_2\})$ are independent,
and the variance of each element in $\frac{\partial L_{k}^{\operatorname{guide}}(\hat{\boldsymbol{\theta}})}{\partial \hat{\boldsymbol{\theta}}^{[l]}} - \frac{\partial L^{\operatorname{guide}}_*}{\partial \hat{\boldsymbol{\theta}}^{[l]}}$ is $\delta$-bounded,
\textit{i.e.},
$\forall k\in \{1, 2, \cdots, n\}, \forall i \in \{1, 2, \cdots, d_1\}, \forall j \in \{1, 2, \cdots, d_2\}, \Var\left(\frac{\partial L_{k}^{\operatorname{guide}}(\hat{\boldsymbol{\theta}})}{\partial \hat{\boldsymbol{\theta}}_{ij}^{[l]}} \right)\le \delta$.
Because $\frac{\partial L^{\operatorname{guide}}_*}{\partial \hat{\boldsymbol{\theta}}^{[l]}}$ is a constant vector, $\frac{\partial L_{k}^{\operatorname{guide}}(\hat{\boldsymbol{\theta}})}{\partial \hat{\boldsymbol{\theta}}_{ij}^{[l]}} - \frac{\partial L^{\operatorname{guide}}_*}{\partial \hat{\boldsymbol{\theta}}_{ij}^{[l]}} \ (k\in \{1, 2, \cdots, n\}, i \in \{1, 2, \cdots, d_1\}, j \in \{1, 2, \cdots, d_2\})$ are also independent.
Further notice that 
\begin{equation}
    \mathbb{E}_{k}\left[ \frac{\partial L_{k}^{\operatorname{guide}}(\hat{\boldsymbol{\theta}})}{\partial \hat{\boldsymbol{\theta}}_{ij}^{[l]}} - \frac{\partial L^{\operatorname{guide}}_*}{\partial \hat{\boldsymbol{\theta}}_{ij}^{[l]}} \right] = 0. \nonumber
\end{equation} 
Therefore,
\begin{equation}
\begin{aligned}
    \mathbb{E}_{k}&\left[\left(\frac{\partial L_{k}^{\operatorname{guide}}(\hat{\boldsymbol{\theta}})}{\partial \hat{\boldsymbol{\theta}}_{ij}^{[l]}} - \frac{\partial L^{\operatorname{guide}}_*}{\partial \hat{\boldsymbol{\theta}}_{ij}^{[l]}} \right)^2\right] \\[4mm]
    &= \Var\left( \frac{\partial L_{k}^{\operatorname{guide}}(\hat{\boldsymbol{\theta}})}{\partial \hat{\boldsymbol{\theta}}_{ij}^{[l]}} - \frac{\partial L^{\operatorname{guide}}_*}{\partial \hat{\boldsymbol{\theta}}_{ij}^{[l]}} \right)\\[4mm]
    &\qquad +
    \left(\mathbb{E}_{k}\left[ \frac{\partial L_{k}^{\operatorname{guide}}(\hat{\boldsymbol{\theta}})}{\partial \hat{\boldsymbol{\theta}}_{ij}^{[l]}} - \frac{\partial L^{\operatorname{guide}}_*}{\partial \hat{\boldsymbol{\theta}}_{ij}^{[l]}} \right]\right)^2 \\[4mm]
    &= \Var\left( \frac{\partial L_{k}^{\operatorname{guide}}(\hat{\boldsymbol{\theta}})}{\partial \hat{\boldsymbol{\theta}}_{ij}^{[l]}} - \frac{\partial L^{\operatorname{guide}}_*}{\partial \hat{\boldsymbol{\theta}}_{ij}^{[l]}} \right)\\[4mm]
    &=\Var\left( \frac{\partial L_{k}^{\operatorname{guide}}(\hat{\boldsymbol{\theta}})}{\partial \hat{\boldsymbol{\theta}}_{ij}^{[l]}} \right) \qquad \left(\text{$\frac{\partial L^{\operatorname{guide}}_*}{\partial \hat{\boldsymbol{\theta}}_{ij}^{[l]}}$ are constants}\right)\\[4mm]
    & \le \delta, \nonumber
\end{aligned}
\end{equation}
\textit{i.e.}, $\mathbb{E}_{k}\left[\left(\frac{\partial L_{k}^{\operatorname{guide}}(\hat{\boldsymbol{\theta}})}{\partial \hat{\boldsymbol{\theta}}_{ij}^{[l]}} - \frac{\partial L^{\operatorname{guide}}_*}{\partial \hat{\boldsymbol{\theta}}_{ij}^{[l]}} \right)^2\right] (k\in \{1, 2, \cdots, n\}, i \in \{1, 2, \cdots, d_1\}, j \in \{1, 2, \cdots, d_2\})$ exist.
By Lemma~\ref{lemma: gap between meta and optimal gradient} we know that $\forall \epsilon > 0$,

\begin{equation}
\begin{aligned}
     &P\left(\frac{1}{n} \sum_{k=1}^n  \sum_{i=1}^{d_1} \sum_{j=1}^{d_2} \left| \frac{\partial L_{k}^{\operatorname{guide}}(\hat{\boldsymbol{\theta}})}{\partial \hat{\boldsymbol{\theta}}_{ij}^{[l]}} - \frac{\partial L^{\operatorname{guide}}_*}{\partial \hat{\boldsymbol{\theta}}_{ij}^{[l]}} \right| \ge \epsilon \right)\\[4mm] 
     &\qquad \le 
     \frac{1}{\epsilon^2} \sum_{k=1}^n  \sum_{i=1}^{d_1} \sum_{j=1}^{d_2} 
     \Var\left( \frac{1}{n} \frac{\partial L_{k}^{\operatorname{guide}}(\hat{\boldsymbol{\theta}})}{\partial \hat{\boldsymbol{\theta}}_{ij}^{[l]}} 
     - \frac{1}{n} \frac{\partial L^{\operatorname{guide}}_*}{\partial \hat{\boldsymbol{\theta}}_{ij}^{[l]}} \right).  \nonumber 
\end{aligned}
\end{equation}
By Theorem 2's assumption, the variance of each element in $\frac{\partial L_{k}^{\operatorname{guide}}(\hat{\boldsymbol{\theta}}^{[l]})}{\partial \hat{\boldsymbol{\theta}}^{[l]}} - \frac{\partial L^{\operatorname{guide}}_*}{\partial \hat{\boldsymbol{\theta}}^{[l]}}$ is $\delta$-bounded,
\textit{i.e.},
$\forall k\in \{1, 2, \cdots, n\}, \forall i \in \{1, 2, \cdots, d_1\}, \forall j \in \{1, 2, \cdots, d_2\}, \Var\left(\frac{\partial L_{k}^{\operatorname{guide}}(\hat{\boldsymbol{\theta}})}{\partial \hat{\boldsymbol{\theta}}_{ij}^{[l]}} \right)\le \delta$. Therefore,
\begin{equation}
\begin{aligned}
    &P\left(\left\| \frac{\partial L_{1:n}^{\operatorname{guide}}(\hat{\boldsymbol{\theta}})}{\partial \hat{\boldsymbol{\theta}}^{[l]}} - \frac{\partial L^{\operatorname{guide}}_*}{\partial \hat{\boldsymbol{\theta}}^{[l]}} \right\|_1  \ge \epsilon \right) \\[4mm]
    &\le
    P\left(\frac{1}{n} \sum_{k=1}^n  \sum_{i=1}^{d_1} \sum_{j=1}^{d_2} \left| \frac{\partial L_{k}^{\operatorname{guide}}(\hat{\boldsymbol{\theta}})}{\partial \hat{\boldsymbol{\theta}}_{ij}^{[l]}} - \frac{\partial L^{\operatorname{guide}}_*}{\partial \hat{\boldsymbol{\theta}}_{ij}^{[l]}} \right| \ge \epsilon \right) \\[4mm]
    &\le  \frac{1}{\epsilon^2} \sum_{k=1}^n  \sum_{i=1}^{d_1} \sum_{j=1}^{d_2} 
     \Var\left( \frac{1}{n} \frac{\partial L_{k}^{\operatorname{guide}}(\hat{\boldsymbol{\theta}})}{\partial \hat{\boldsymbol{\theta}}_{ij}^{[l]}} 
     - \frac{1}{n} \frac{\partial L^{\operatorname{guide}}_*}{\partial \hat{\boldsymbol{\theta}}_{ij}^{[l]}} \right)\\[4mm]
    &= \frac{1}{\epsilon^2} \sum_{k=1}^n  \sum_{i=1}^{d_1} \sum_{j=1}^{d_2} 
     \Var\left( \frac{1}{n} \frac{\partial L_{k}^{\operatorname{guide}}(\hat{\boldsymbol{\theta}})}{\partial \hat{\boldsymbol{\theta}}_{ij}^{[l]}} \right) \qquad \\[4mm]
     & \qquad\qquad\qquad \left(\text{because $\frac{\partial L^{\operatorname{guide}}_*}{\partial \hat{\boldsymbol{\theta}}_{ij}^{[l]}}$ are constants}\right) \\[4mm]
    & \le \frac{1}{\epsilon^2}  \sum_{k=1}^n  \sum_{i=1}^{d_1} \sum_{j=1}^{d_2} \frac{1}{n^2}\delta \\[4mm]
    &= \frac{d_1 d_2 \delta}{\epsilon^2} \frac{1}{n}.   \nonumber 
\end{aligned}
\end{equation}

%% file: chapters/8_appendixB.tex
\textbf{Software.} We use the following software versions: 
\begin{itemize}
    \item Python 3.9.11
    \item Pytorch  1.11.0+cu113~\cite{pytorch}
    \item Gym  0.23.1~\cite{openai_gym}
    \item MuJoCo 2.1.3\footnote{License information: \url{https://www.roboti.us/license.html}}~\cite{mujoco}
    \item mujoco-py 2.1.2.14
    \item d4rl  1.1~\cite{Fu2020D4RLDF}
\end{itemize}
The Gym locomotion-v2~\cite{mujoco,openai_gym} and robotic manipulation adroit-v1~\cite{adroit} versions in the D4RL benchmark~\cite{Fu2020D4RLDF} datasets are adopted.

\textbf{Hyperparameters.}
We consider several state-of-the-art methods as baselines, including TD3+BC~\cite{td3+bc}, SAC+BC (a variant of TD3+BC substituting TD3~\cite{td3} with SAC~\cite{sac}), CQL~\cite{cql}, and IQL~\cite{iql}. 
Our implementations of TD3+BC\footnote{\url{https://github.com/sfujim/TD3_BC}}~\cite{td3+bc},
CQL\footnote{\url{https://github.com/aviralkumar2907/CQL}}~\cite{cql}, 
and IQL\footnote{\url{https://github.com/rail-berkeley/rlkit/tree/master/examples/iql}}~\cite{iql} are based on respective published papers and author-provided implementations from GitHub. For SAC+BC, we select the optimal hyperparameters by grid search.
For fair comparison, GORL keeps the same hyperparameters as that of the corresponding base algorithms.
More Details of hyperparameters are provided in Table~\ref{table:td3_hyp}, \ref{table:cql_hyp}, \ref{table:sac_hyp}, and \ref{table:iql_hyp}.

\begin{table}[ht]
\caption{Hyperparameters of TD3+BC~\cite{td3+bc} with GORL on locomotion / adroit datasets.}
\centering
\begin{tabular}{cll}
\toprule
& Hyperparameter & Value \\
\midrule
\multirow{9}{*}{TD3 Hyperparameters} & Optimizer & Adam~\cite{adam} \\
                                      & Critic learning rate & 3e-4 \\
                                      & Actor learning rate  & 3e-4 \\
                                      & Mini-batch size      & 256 \\
                                      & Discount factor      & 0.99 \\
                                      & Target update rate   & 5e-3 \\
                                      & Policy noise         & 0.2 \\
                                      & Policy noise clipping & (-0.5, 0.5) \\
                                      & Policy update frequency & 2 \\
\midrule
\multirow{6}{*}{TD3 Architecture}     & Critic hidden dim    & 256        \\
                                      & Critic hidden layers & 2          \\
                                      & Critic activation function & ReLU \\
                                      & Actor hidden dim     & 256        \\
                                      & Actor hidden layers  & 2          \\
                                      & Actor activation function & ReLU \\
\midrule
\multirow{2}{*}{TD3+BC Hyperparameters}  & $\lambda$             & 2.5 / 0.1 \\
                                         & State normalization            & True\\
\midrule
\multirow{4}{*}{GORL Hyperparameters}            
                                      & Guiding-net learning rate    & 1e-5      \\
                                      & Guiding-net update frequency  & 500       \\
                                      & Guiding-data size       & 200       \\
                                      & Guiding-data mini-batch size     & 20       \\
\midrule
\multirow{4}{*}{Guiding-Net Architecture}            
                                      & Hidden dim          & 100    \\
                                      & Hidden layers       & 1      \\
                                      & Activation function & Sigmoid \\
                                      % & Weight normalization & 0 / 20 \\
\bottomrule
\end{tabular}
\vspace{12pt}
\label{table:td3_hyp}
\end{table} 

\begin{table}[ht]
\caption{Hyperparameters of CQL~\cite{cql} with GORL on locomotion / adroit datasets.}
\centering
\begin{tabular}{cll}
\toprule
& Hyperparameter & Value \\
\midrule
\multirow{6}{*}{CQL Hyperparameters}  & Optimizer & Adam~\cite{adam} \\
                                      & Policy learning rate & 1e-4 \\
                                      & Mini-batch size      & 256 \\
                                      & Lagrange thresh      & -1.0 \\
                                      & Min q weight        & 5.0 / 1.0 \\
                                      & Min q version        & 3 / 2 \\
\midrule
\multirow{4}{*}{GORL Hyperparameters}            
                                      & Guiding-net learning rate    & 1e-5      \\
                                      & Guiding-net update frequency  & 200 / 100       \\
                                      & Guiding-data size       & 200       \\
                                      & Guiding-data mini-batch size     & 20       \\
\midrule
\multirow{4}{*}{Guiding-Net Architecture}            
                                      & Hidden dim          & 100    \\
                                      & Hidden layers       & 1      \\
                                      & Activation function & Sigmoid \\
\bottomrule
\end{tabular}
\label{table:cql_hyp}
\end{table} 

\begin{table}[ht]
\caption{Hyperparameters of SAC+BC~\cite{td3+bc} with GORL on locomotion / adroit datasets.}
\centering
\begin{tabular}{cll}
\toprule
& Hyperparameter & Value \\
\midrule
\multirow{8}{*}{SAC Hyperparameters}  & Optimizer & Adam~\cite{adam} \\
                                      & Critic learning rate & 3e-4 \\
                                      & Actor learning rate  & 3e-4 \\
                                      & Mini-batch size      & 256 \\
                                      & Discount factor      & 0.99 \\
                                      & Target update rate   & 5e-3 \\
                                      & Alpha target         & 0.2 \\
                                      & Policy update frequency & 2 \\
\midrule
\multirow{6}{*}{SAC Architecture}     & Critic hidden dim    & 256        \\
                                      & Critic hidden layers & 2          \\
                                      & Critic activation function & ReLU \\
                                      & Actor hidden dim     & 256        \\
                                      & Actor hidden layers  & 2          \\
                                      & Actor activation function & ReLU \\
\midrule
\multirow{2}{*}{SAC+BC Hyperparameters}  & $\lambda$             & 2.5 / 0.1 \\
                                         & State normalization            & True\\
\midrule
\multirow{4}{*}{GORL Hyperparameters}            
                                      & Guiding-net learning rate    & 1e-5      \\
                                      & Guiding-net update frequency  & 500       \\
                                      & Guiding-data size       & 200       \\
                                      & Guiding-data mini-batch size     & 20      \\
\midrule
\multirow{4}{*}{Guiding-Net Architecture}            
                                      & Hidden dim          & 100    \\
                                      & Hidden layers       & 1      \\
                                      & Activation function & Sigmoid \\
\bottomrule
\end{tabular}
\vspace{6pt}
\label{table:sac_hyp}
\end{table} %

\begin{table}[ht]
\caption{Hyperparameters of IQL~\cite{iql} with GORL on locomotion / adroit datasets.}
\centering
\begin{tabular}{cll}
\toprule
& Hyperparameter & Value \\
\midrule
\multirow{6}{*}{IQL Hyperparameters}  & Optimizer & Adam~\cite{adam} \\
                                      & Policy learning rate & 3e-4 \\
                                      & Mini-batch size      & 256 \\
                                      & Dropout rate         & 0.0 / 0.1 \\
                                      & Beta        & 3 / 0.5 \\
                                      & Quantile & 0.7 \\
\midrule
\multirow{4}{*}{GORL Hyperparameters}            
                                      & Guiding-net learning rate    & 1e-5      \\
                                      & Guiding-net update frequency  & 200 / 1       \\
                                      & Guiding-data size       & 200       \\
                                      & Guiding-data mini-batch size     & 20      \\
\midrule
\multirow{4}{*}{Guiding-Net Architecture}            
                                      & Hidden dim          & 100    \\
                                      & Hidden layers       & 1      \\
                                      & Activation function & Sigmoid \\
\bottomrule
\end{tabular}
\vspace{6pt}
\label{table:iql_hyp}
\end{table}